\documentclass[11pt]{article}

\usepackage{libertine}
\usepackage[T1]{fontenc}
\usepackage{amsmath,amsthm}
\usepackage{booktabs}
\usepackage{multirow}
\usepackage{graphicx}
\usepackage{xcolor}

\usepackage{arxivs}
\usepackage{amsfonts,amssymb}

\usepackage{cite}
\usepackage{hyperref}
\usepackage{textcomp}
\usepackage{tabularx}
\usepackage{caption}
\usepackage{subcaption}
\usepackage{diagbox}

\newtheorem{definition}{Definition}
\newtheorem{theorem}{Theorem}
\newtheorem{proposition}{Proposition}
\usepackage{hyperref}

\usepackage{bm}

\newcommand{\Eb}{\mathbb{E}}

\newcommand{\Rb}{\mathbb{R}}

\newcommand{\Ac}{\mathcal{A}}

\newcommand{\Fc}{\mathcal{F}}
\newcommand{\Jc}{\mathcal{J}}
\newcommand{\Kc}{\mathcal{K}}
\newcommand{\Lc}{\mathcal{L}}
\newcommand{\Pc}{\mathcal{P}}
\newcommand{\Sc}{\mathcal{S}}

\newcommand{\Yc}{\mathcal{Y}}
\newcommand{\sys}{\text{sys}}

\newcommand{\kc}{\kappa}

\newcommand{\one}{\mathbb{I}}
\newcommand{\onen}{\mathbf{1}}

\newenvironment{tightitemize}{%
    \list{{\textup{$\bullet$}}}{\settowidth\labelwidth{{\textup{\qquad}}}
    \leftmargin\labelwidth \advance\leftmargin\labelsep
    \parsep 0pt plus 1pt minus 1pt \topsep 3pt \itemsep 3pt
    }}{\endlist}

\title{Robust and Fair Multi-class Classification via Systemic Risk}
\author{Darinka Dentcheva \\
        Department of Mathematical sciences \\
        Stevens Institute of Technology \\
        Hoboken, NJ 07030, USA \\
        \texttt{darinka.dentcheva@stevens.edu}\\
        \And 
        Xiangyu Tian \\
        Department of Mathematical sciences \\
        Stevens Institute of Technology \\
        Hoboken, NJ 07030, USA \\
        \texttt{xtian9@stevens.edu}
        }

\date{}

\begin{document}

\maketitle

\begin{abstract}
We develop a new multi-class classification framework based on the theory of coherent risk measures and systemic risk. The proposed approach is suitable for problems when the data is noisy, corrupted, or scarce relative to the dimension of the problem. The paper provides the foundation of models using systemic risk and their application in the context of linear and kernel-based multi-class problems. A system-theoretic approach with non-linear aggregation of contextual risk is proposed, which leads to a two-stage stochastic programming problem. A risk-averse regularized decomposition method is designed to solve the problem. The computational effort grows linearly with the number of points in the data set. We use a popular multi-class method to construct a risk-averse counterpart methods and to serves as a benchmark of the performance of the new methods. 
We demonstrate that the application of the proposed framework  provides robustness with respect to changes in the probability distributions involved arising from noisy or corrupted data, mislabeling, or other. Furthermore, we show how to use the systemic measures of risk to enforce equal opportunity fairness at the same time. Analysis and experiments regarding the fairness of the proposed models are carefully conducted and benchmarked to recent methods designed to address robustness and fairness simultaneously. Our experiments demonstrate that the proposed risk-averse classification methods are robust in the presence of unreliable training data and perform better on unknown data than the methods minimizing an expected loss function of classification errors. Furthermore, the performance improves when the number of classes increases.
\end{abstract}


\section{Introduction}

Classification is one of the most important machine learning methods that can be used in many real-life situations, such as fraud detection, health anomalies, natural disaster detection, etc. Our paper presents a new risk-averse design of the classification model. 

Robustness in classification design has become an important area of study due to the increasingly complex data being analyzed in applications where misclassification may have very serious consequences. The objective of robust classification is to develop models that are stable and accurate even when data is incomplete, noisy, or with polluted labels. In adversarial situations, mislabeling typically occurs. 

Various approaches to robustness with respect to polluted data exist in the literature and quite a few papers deal with this issue.  Providing an exhaustive survey is beyond the scope of this paper; we only mention several approaches and some associated works. Methods investigated by a number of researchers employ robust optimization with various constructions of uncertainty sets; we refer to \cite{ghaoui1997robust,bertsimas2018robust,bi2005support}. This method does not only work with noisy features, but also works when the labels are polluted. Another approach is based on chance constraints, e.g., \cite{lanckriet2002robust,bental2011chance}. 
Preprocessing the data is also a common option. Such as the work in \cite{northcutt2017learning}, it has proposed a method to deal with the noisy labels in binary classification problems. The method can take a mislabeled training set, and find a subset of it that is more likely to be correctly labeled. In \cite{patrini2017making}, the robustness of the classification model against the noisy labels in deep neural networks is investigated. Their method is compatible with multi-class setting. The preprocessing methods have the potential to be combined with other methods that focus on the model formulation.

Other work addresses robustness to corrupted data by tailoring the risk estimation and the loss function (\cite{barron2019general,ghosh2017robust,sudre2017generalised}). In \cite{lecue2020robust}, the traditional empirical risk, which is the estimator of the expectation of the loss function, is replaced by the mean of the medium value estimator. The authors argue that the minimizer of that estimator is less sensitive to the corruption of the training data than the empirical risk minimizer, and hence, it is more robust. The work in \cite{sypherd2022tunable} proposed a 
tunable $\alpha$-loss as an effective generalization of the cross entropy loss, which can be either more robust or more sensitive depending on the actual application. 

Other influential work regarding robust classification design changes the framework of the learning process, see \cite{freund1997decision,breiman2001random,friedman2001greedy}. These approaches are typically based on combinations of methods. For example, AdaBoost \cite{freund1997decision} iteratively reweights the training data, giving more weight to the instances that were misclassified in previous rounds, which can be interpreted as data preprocessing. Additionally, a form of adaptive loss function is employed, where more emphasis is put on the samples that are harder to classify. 

In our paper, we propose a new way to estimate the misclasification risk which is particularly suitable for multi-class scenarios. Instead of the empirical risk, we use coherent risk measures for \emph{systemic} risk. 
While coherent risk measures are widely used in finance and insurance, the theory and numerical methods for systemic measures of risk are less developed. As risk is not additive, the systemic measures of risk aim at risk aggregation and, respectively, risk allocation to the units of a system in a coherent way.

 Coherent measures of risk have been used in classification as well as in other contexts of statistical learning because they have a built-in robustness while being amenable to efficient numerical treatment. In the context of classification, coherent measures of risk have been used mainly in a binary classification scenario in \cite{vitt2019risk, gotoh2017support,norton2017soft}. In \cite{gotoh2017support}, the authors use an overall risk function to measure the risk of the entire system while in \cite{vitt2019risk}, the authors proposed a new family of loss functions, which apply different risk measures on different classes. They argue that in different tasks, the importance of each class should be different, therefore treating them differently can lead to a more flexible and more adequate model. 

 The new framework for a multi-class classification that is proposed in this paper has build-in robustness corresponding to distributionally robust methods with implicitly defined ambiguity sets. This results in a smaller computational burden while still providing robustness to polluted data, mislabeling, and general perturbations of the distributions involved. Furthermore, the proposed approach has the potential to properly evaluate and enforce fairness at the same time. 
The extant literature  points to a trade-off between fairness and model performance (e.g., \cite{rychener2022metrizing}). We note however that compared to the baselines, our method sacrifices less performance to achieve better fairness. We shall make the relevant notions precise in due course.
 While our framework is generally aligned with the ideas of risk-sharing in classification models similar to \cite{vitt2019risk}, our proposal is more general and unifies the approaches to risk-sharing as it is based on an axiomatic foundation (\cite{almen2024risk}). Additionally, we have devised a new decomposition method for solving the resulting optimization problem and have shown its convergence. We have demonstrated that the numerical effort 
 required by the proposed method is significantly smaller than the effort needed by a standard optimization solver. 

We demonstrate the proposed methodology on a risk-averse counterpart of a popular multi-class classification method. In the comparison, we used the empirical estimators of the risk measures and we experimented with corrupted data. The risk-averse model shows statistically significant better performance and better generalization towards the unknown. Furthermore, the risk-averse design is more helpful in diverse situation involving more complex multi-class setting. The numerical results show that the performance advantage of the proposed risk-averse model over the risk-neutral model becomes more pronounced when more classes are involved. 
Further, we demonstrate that equal opportunity fairness can be addressed with negligible decrease of accuracy and with preservation of robustness at the same time with no additional computational burden.

\section{Coherent Measures of Risk for Random Vectors}

The risk of a loss function can be evaluated using a univariate coherent measure of risk according to the widely accepted axiomatic framework that is proposed in \cite{artzner1999coherent} and further analyzed in \cite{Delbaen,Follmer,RS:2005,RuSh:2006a,PflRom:07,DDARriskbook}, and many others work.
We refer to \cite{DDARriskbook} for an extensive treatment of risk measures and stochastic optimization with such measures. 

    Let $\Lc_p(\Omega, \Fc, P;\Rb^N)$ be the space of random vectors with realizations in $\Rb^N$, defined on the probability space $(\Omega,\Fc, P)$, which have finite $p$-th moments, $p \in [1, \infty)$, and are indistinguishable on events with zero probability. We shall assume that the random variables represent losses (i.e., small outcomes are preferred) since we plan to identify them with classification errors. 
    
    When we deal with scalar-valued random variables, then a lower semi-continuous functional $\varrho: \Lc_p(\Omega, \Fc, P; \Rb) \to \Rb\cup\{+\infty\}$ is called a \textit{coherent risk measure} if it is convex, positively homogeneous, monotonic with respect to the a.s. comparison of random variables, and satisfies the following translation property:
    \[
        \varrho[Z + a] = \varrho[Z] + a \text{ for all } Z \in \Lc_p(\Omega, \Fc, P),\; a \in \Rb. 
    \] 
    While the theory and methods for risk measures of scalar-valued random variables are well-developed, less work is associated with high dimensional risks, that is, measures of risk for random vectors. The need for special attention to random vectors arises in part from the challenge that univariate coherent (or convex) measures of risk are not additive except for the expected value functional.

    We denote the $N$-dimensional vector, whose components are all equal to one by $\onen$, and the random vector with realizations equal to $\onen$ by $\one$. We adopt the following definition, introduced in \cite{almen2024risk}, see also \cite{DDARriskbook}: 
    \begin{definition}
    \label{d:riskonvectors}
        A lower semi-continuous functional $\varrho: \Lc_p(\Omega, \Fc, P;\Rb^N) \to \Rb\cup\{+\infty\}$ is a systemic coherent risk measure with preference to small outcomes if it has the following properties:
        \begin{tightitemize}
            \item[A1.] \emph{Convexity}: For all $X, Y \in \Lc_p(\Omega, \Fc, P;\Rb^N)$ and for all $ \alpha \in (0,1)$, 
                $ ~ \varrho[\alpha X + (1-\alpha)Y] \leq \alpha \varrho[X] + (1 - \alpha)\varrho[Y]. $
            \item[A2.] \emph{Monotonicity}: For all $X, Y \in \Lc_p(\Omega, \Fc, P;\Rb^N)$, if $X_i \geq Y_i$ $P$-a.s. for all components $i = 1, \dots, N$, then $\varrho[X] \geq \varrho[Y]$.
            \item[A3.] \emph{Positive homogeneity}: For all $X \in \Lc_p(\Omega, \Fc, P;\Rb^N)$ and $t > 0$, we have $\varrho[tX] = t\varrho[X]$.
            \item[A4.] \emph{Translation}: For all $X \in \Lc_p(\Omega, \Fc, P;\Rb^N)$ and for all $a \in \Rb$,  $\varrho[X + a\one] = \varrho[X] + a\varrho[\one].$
        \end{tightitemize}
    \end{definition}
    The concept of coherent risk measures is also commonly used in the field of finance, where the definition is slightly different in the way that the random quantity of interest is usually a payoff instead of a loss, i.e., large outcomes are preferred. In our paper, the random quantity would be the classification errors. It is shown in \cite{almen2024risk,DDARriskbook} that if the systemic risk measure $\varrho$ is proper, lower semicontinuous, and satisfies those axioms, then it can be represented as follows:
    \begin{equation}
        \varrho[X] = \sup_{\zeta \in \Ac_\varrho} \langle \zeta, X \rangle_{\Lc} = \sup_{\zeta \in \Ac_\varrho} \int_\Omega \zeta(\omega)^\top X(\omega)\, P(d\omega), 
        \label{general dual}
    \end{equation}
    where $\langle \cdot,\, \cdot \rangle_{\Lc}$ denotes the dual pairing between $\Lc_p(\Omega, \Fc, P;\Rb^N)$  and $\Lc_q(\Omega, \Fc, P;\Rb^N)$ with $\frac{1}{p} + \frac{1}{q} =1$ and for $p=1$, $q=\infty$.
    The set $\Ac_\varrho\subset \Lc_q(\Omega, \Fc, P;\Rb^N)$ is the convex subdifferential $\partial \varrho(0)$ of the risk measure and it satisfies:
    \[
        \Ac_\varrho (P) \subseteq \Big\{ \zeta \in \Lc_q(\Omega, \Fc, P;\Rb^N) ~|~ \zeta \geq 0 \text{ a.s.},\,  \int_\Omega \onen^\top \zeta(\omega) dP(\omega) \big\rangle =r  \Big\},
    \]
    where $r\in\Rb$ is a constant. 

    A systemic measure of risk $\varrho$ is \emph{normalized} if $\varrho[\one] = 1$, in which case  $r =1$ for all $\zeta \in \Ac_\varrho(P)$.
    This entails that for all $\zeta \in \Ac_\varrho(P)$, $\zeta P$ can be interpreted as a probability measure on the space $\Omega \times \{ 1, 2 ,\dots, N \}$.  In the special case when $N=1$, we obtain the widely used dual representation of coherent measures of risk for scalar-valued random variables 
    \begin{equation}
        \varrho[X] = \sup_{\frac{dQ}{dP} \in \Ac_\varrho(P)} \Eb_Q [ X ],
        \label{eq:scalar dual}
    \end{equation}
    where $\frac{dQ}{dP}$ is the Radon-Nikodym derivative of the measure $Q$ with respect to the reference measure $P$. The dual representation \eqref{general dual} demonstrates the link between minimizing a coherent measure $\varrho(X)$ and distributionally robust optimization (DRO).  The DRO problem with an ambiguity set $\Ac \subset\Pc (\Omega)$ and with a loss function $L(X(\vartheta, \omega))$, where $\vartheta$ is the decision vector, is formulated as follows:
    \begin{equation}
    \min_{\vartheta}\ \sup_{Q\in\Ac}\ \Eb_Q \big[L (X(\vartheta,\omega))\big].
    \label{eq:DRO}
    \end{equation}
    Problem \eqref{eq:DRO} corresponds to the following risk-averse model:
    \begin{equation*}
    \min_{\vartheta}\varrho\big[L (X(\vartheta,\omega))\big] = \min_{\vartheta} \sup_{\frac{dQ}{dP} \in\Ac_\varrho(P)}\ \Eb_Q \big[L (X(\vartheta,\omega))\big].
    \end{equation*}
    As an example, consider the widely used coherent risk measure for scalar-valued random variables called mean - upper semideviation; it is defined as follows
    \[
    \varrho(X) = \Eb[X]+\kappa\big\|(X-\mathbb{E}[X])_{+}\big\|_p,\; \text{ where  } \kappa\in[0,1].
    \]
    Here $\|\cdot\|_p$ stands for the $p$-norm in $\Lc_p(\Omega, \Fc, P;\Rb)$  and $(a)_+ = \max(0,a)$. 
    The mean - upper semideviation has the dual representation \eqref{eq:scalar dual} with the dual set 
    \begin{equation}
    \label{e:dualset-mean-sd}
    \Ac_\varrho(P) =\left\{Q\ll P:\;\tfrac{dQ}{dP}=\one+\zeta-\mathbb{E}[\zeta]\one,\; \|\zeta\|_q \leq \kappa,\; \zeta \geq 0\right\}, 
    \end{equation}
    i.e, $Q$ are probability measures with densities $\xi$ w.r.to $P$ given by $\xi = \one+\zeta-\mathbb{E}[\zeta]\one.$
    Notice that the parameter $\kappa$ controls the risk-averse level of the objective; larger value of $\kappa$ allows larger weight-differences in $\zeta$, leading to a larger ambiguity set. For fruther examples, we refer to \cite{DDARriskbook}.

    \section{The risk-averse classification problem}

   We consider labeled data consisting of $N$ subsets $\Sc_1,\Sc_2,\dots,\Sc_N$ of $n$-dimensional vectors $x\in\Rb^n$ and denote $\Jc = \{1,\dots, N\}$.
 The labels of the data points in $\Sc_i$ are denoted by $y_i$, and $y_i=i$ stands for the $i$-th class. The size of each class is the cardinality $|\Sc_i|=m_i$, $i\in\Jc$. The data points in the subset $\Sc_i$ are the observations of the $n$ features of the $i$-th class. We identify the $N$-dimensional random vector $Z$ with the classification error (measured in some way) associated with a given classifier, that is, $Z^{i}$ represents the classification error of class $i$, $i\in\Jc$; $Z^i$ has $m_i$ realizations: one for each data point in the class $\Sc_i$.   

As already mentioned, the risk measures are non-additive and the overall risk of the system does not equal to the direct sum of the risk of each class. In many applications, the system's risk is computed by linear aggregation of the risk of individual components assigning appropriate weights to each of them:
\begin{equation*}
\varrho_{\rm s}(Z) = \sum_{i\in\Jc} c_i \varrho_i(Z^i).
\end{equation*}
The argument to do so is that a Pareto-efficient risk allocation to each class $i$ can be obtained by the use of a scalarization of the vector of risks. Whenever $c_i\geq 0$ and $\sum_{i=1}^{N} c_i =1$ the risk allocation vector $\mathcal{R} = \big(\varrho_1[Z^1],\dots, \varrho_N[Z^N]\big) $ 
is obtained as a minimal element of the set of feasible risk allocations. The advantage of this method is that its calculation is very efficient. However, it is not always clear how to adjust the weight of each class especially when the number of classes increases. We propose a new way of scalarization by using a nonlinear aggregation of the risks associated with the individual classes. 

A finite probability space $(\Omega_N,\Fc_N, c)$ is given by $\Omega_N = \{1,\dots,N\}$, $c$ is a probability mass function, and $\Fc_N$ contains all subsets of $\Omega_N$. Given a collection of $N$ measures of risk $\varrho_i:\Lc_p(\Omega,\Fc, P;\Rb)\to\Rb$, $i\in\Jc$,  we associate with the random error vector $Z$, a random variable $R_Z$ on the space $\Omega_N$ as follows. The realizations of $R_Z$ are 
    \begin{equation*}
        \label{R_Z}
        R_Z(i) = \varrho_i(Z^i),\quad i\in\Jc. 
    \end{equation*}
 Let a coherent measure of risk $\varrho_{\rm o}:\Lc_\infty(\Omega_N,\Fc_N, c)\to\Rb$  be chosen as an aggregator. The measure of systemic risk 
     $\varrho_{\rm sys}:\Lc_p(\Omega, \Fc, P; \Rb^N)\to\Rb$ is defined as:
    \begin{equation}
         \varrho_{\rm sys} (Z) = \varrho_{\rm o}(R_Z). 
         \label{syst_risk2} 
    \end{equation} 
This type of measure satisfies the axioms in Definition~\ref{d:riskonvectors} as shown in \cite{almen2024risk}.
With our goals in mind, we shall pay special attention to the following classification risk evaluation. 
We choose as $\varrho_{\rm o}$ the mean-upper-semideviation risk measure of order $q$ ($q\geq 1$) and evaluate all components of the error vector $Z$ by the same law-invariant coherent measure of risk $\varrho(\cdot)$. The description of the total risk evaluation is the following: 
    \begin{equation}
    \label{d:outer-MSD-p}
  \varrho_\sys(Z)  = \sum_{i\in\Jc} c_i \varrho(Z^i) + \kappa \bigg(\sum_{i\in\Jc} c_i \Big(\varrho(Z^i) - \sum_{j\in\Jc} c_j \varrho(Z^j) \Big)_+^q\bigg)^\frac{1}{q}     
    \end{equation} 
    with $\kappa\in[0,1]$.
    This representation shows that the individual risks of the components are aggregated with an additional penalty on the deviation of the individual risks from that average risk. This property is crucial to our treatment of fairness in classification. 

Another example would be to define $\varrho_{\rm o}$ as a convex combination of the expected value and the Average Value-at-Risk at some level $\alpha$; again all components of the error vector $Z$ may be evaluated by the same law-invariant coherent measure of risk $\varrho(\cdot)$. 
Then for any $\kappa \in [0,1]$, the systemic measure of risk takes on the form:
    \begin{equation}
    \label{e:avar-comb}
        \varrho_\sys (Z)  = (1-\kappa) \sum_{i\in\Jc} c_i \varrho (Z^i) + 
                          \kappa \inf_{\eta \in \Rb} \Big\{  \eta + \frac{1}{\alpha} \sum_{i\in\Jc} c_i\big(\varrho (Z^i) - \eta\big)_+ \Big\}.
    \end{equation}
    Here, the infimum with respect to $\eta \in \Rb$ is taken over the individual risks of the components $\varrho (Z^i)$, $i\in\Jc$. Hence, this method of aggregation imposes additional penalty for the components whose risk exceeds a certain treshhold.

 We now turnm to the classification problem. Very powerful and the most widely used classification methods are the vaerious versions of support vector machines.  Most popular approaches for multi-class classification use techniques called One-vs-All and One-vs-One.
For $N$ classes, the One-vs-All method identifies $N$ separate binary classifiers such that the $i$-th classifier discriminates between the $i$-th class and all the rest of the data. After the training process, the prediction will be the class with the highest score among the $N$ classifiers. The  One-vs-One method identifies $\binom{N}{2}$ classifiers for every pair of the $N$ classes. After the training process, for a new data point, all the $\binom{N}{2}$ classifiers would be applied,  and the prediction will be the class that has been chosen most frequently by those classifiers. 
Many designs exist that do not require solving multiple optimization problems to classify multiple classes.
We consider as benchmark methods those presented in  \cite{crammer2001algorithmic, lee2004multicategory, weston1999support}, which solve only one optimization problem. 

The Crammer-Singer method is a very well-known and frequently cited classification method. The method determines $N$ linear classifiers $\psi_i:\Rb^n\to\Rb$,  $\psi_i(x) = \langle v^i, x \rangle - \gamma_i$, $i\in\Jc$, by solving the following problem:
\begin{align}
\min_{v,\gamma,Z} \quad & \sum_{i\in\Jc} \left( \frac{1}{m_i} \sum_{\ell=1}^{m_i} z^i_\ell \right) + \sigma \sum_{i\in\Jc} \|v^i\|^2  \label{eq:cs-rn-obj}\\
\text{subject to} \quad & z^i_\ell \geq  \psi_j(x^i_\ell)  - \psi_i (x^i_\ell)  + 1 \quad   i\in\Jc,\; j \in\Jc\setminus\{i\}, \,  \ell = 1, \ldots, m_i \label{eq:cs-rn-errors}\\
& Z^i \geq 0 \quad  i \in\Jc \label{eq:cs-rn-nonneg}.
\end{align}
Here $\sigma > 0$ is a small regularization parameter. In this formulation, $Z^i$ stands for the random variable with realizations $z^i_\ell$, $\ell = 1, \ldots, m_i.$ 
The problem determines $N$ classifiers such that ideally 
\begin{equation}
\label{ideal}
\psi_i(x) > \max_{j\in \Jc\; j\neq i} \psi_j(x)
\end{equation} 
for all observed data points $x$ in the $i$-th class. The data points $x$ in the training set $\Sc_i$, which violate \eqref{ideal} are identified by the positive values of the corresponding realizations of $Z^i$.
In the training problem, we can view the objective function as the minimization of the expected sum of maximal violations for all classes in a soft-margin formulation. 
The Crammer–Singer method fits to our perspective that treats the multiclass problem as a whole system.

A risk-averse version of problem \eqref{eq:cs-rn-obj}--\eqref{eq:cs-rn-nonneg} using a simple linear aggregation of the risks results in the following classification problem 
\begin{equation}
\begin{aligned}
\min_{v,\gamma,Z} \quad & \sum_{i\in\Jc} c_i \varrho_i [Z^i]+\sigma \|v\|^2\quad
\text{subject to} \quad  \eqref{eq:cs-rn-errors}, \eqref{eq:cs-rn-nonneg}.
\end{aligned}
\label{eq:cs-ra}
\end{equation}

We assume that we observe data from class $i$ with probability $c_{i}$, $i\in \Jc$. 
The probabilities $c_i$ can be given by prior knowledge of the class distribution or simply use the sample size of each class as an estimate, i.e., $c_i=\frac{m_i}{\sum_{j\in\Jc} {m_j}}$. In order to de-clutter the presentation, we introduce the shorthand notation:  
\begin{gather*}
v\in\Rb^{nN} \text{ has  components } v^i\in\Rb^n,\quad \gamma\in \Rb^N \text{ with components } \gamma_i, \\
\vartheta=(v,\gamma)\in\Rb^{N(n+1)}\;\text{ and }\; \Jc^{-i} = \Jc\setminus \{i\},\quad i\in\Jc.
\end{gather*}
It is shown in \cite{almen2024risk} that the systemic measure of risk can be a maximum of a family of coherent univariate risk measures on $\Lc_{\infty}(\Omega_N,\Fc_N, c)$ instead of a single measure $\varrho_{\rm o}$. We note that the expectation is the simplest coherent measure of risk that can be an aggregator $\varrho_{\rm o}$. Therefore, the approach with systemic risk  provides \emph{a unifying view} on risk-allocation to multiple system's components and to classification in multi-class scenario. We shall observe also that a nonlinear aggregation of classification errors based on the axiomatic framework delivers classifiers with better properties. 

From this new perspective, the classification problem can be formulated as a two-stage stochastic optimization problem. The first stage problem is:
\begin{equation}
\begin{aligned}
\min_{\vartheta} \quad & \varrho_{\rm o}[R(\vartheta)] + \sigma \left\|v\right\|^2.
\end{aligned}
\label{p:first-stage}
\end{equation}
In the first stage, we decide the best classifiers $\psi_i$, $i\in \Jc$, to calculate the optimal systemic risk and $R(\vartheta)$ is the random variable providing the total risk; the risk components for each class are calculated at the second-stage. 
We observe that the proper evaluation of the classification error requires to include the constraints $ \|v^i\|=1$ for all $ i\in\Jc$.  While it is possible to solve the resulting non-convex problem, we have kept the regularization terms  $\sigma \left\|v\right\|^2$ instead. 
In the second-stage, a coherent risk measure $\varrho_i(\cdot)$ serves as the objective of the $i$-th scenario ($i$-th class).  We calculate the risk for the $i$th class by the problem:
\begin{equation}
\label{p:second-stage}
\begin{aligned}
   R_i(\vartheta) = \min_{Z^i\geq 0} \quad & \varrho_i[Z^i]\quad
    \text{s.t.} \quad  z^i_\ell \geq \langle v^j, x^i_\ell\rangle - \langle v^i, x^i_\ell\rangle +1 \quad 
     j \in \Jc^{-i}, \;  \ell = 1, \ldots, m_i.
\end{aligned}
\end{equation}
We shall assume throughout the entire paper that all risk measures involved in the classification problem take finite value for bounded random variables. The two-stage problem \eqref{p:first-stage}--\eqref{p:second-stage} is equivalent to a large scale one-stage problem, which is formulated 
using auxiliary variables $q_i\in\Rb$, $i\in\Jc$. The formulation is the following:
\begin{align}
\min_{v,\gamma,Q,Z} \quad & \varrho_{\rm o}[Q] + \sigma \sum_{i\in\Jc} \left\|v^i\right\|^2, \label{one-stage-obj}\\
\text{s.t. }\; & 
   q_i \geq  \varrho_i(Z^i),\quad i\in\Jc, \label{one-stage-qi} \\
     & z^i_\ell \geq \psi_j(x^i_\ell) - \psi_i (x^i_\ell) +1 \quad 
     i \in\Jc, \;  \ell = 1, \ldots, m_i, \; \; j\in \Jc^{-i},  \label{one-stage-ziell}\\
    & Z^i \geq 0 \quad i\in\Jc.  \label{one-stage-Znon-neg}
\end{align}
Here $Q$ is a random variable with realizations $q_i$, $i\in\Jc$. 

When using this type of systemic risk measure we no longer need to worry about selection of weights to reflect the importance of each class unless a pertinent reason requires placing an emphases on specific classes. In the latter case, we may increase their visibility by modifying the scalarization $c$.

\section{Kernel methods for risk-averse classification}

When using kernel-based methods, we assume that a pre-Hilbert space is given, i.e., a space $\mathcal{Z}$, where the inner product is defined and a reproducing kernel 
$K:\mathcal{Z}\times\mathcal{Z}\to\Rb$ exists. More precisely, a non-linear mapping $\varphi: \Rb^n\to \mathcal{Z}$  exists such that
$K(x,x') = \langle \varphi (x),  \varphi (x')\rangle_\mathcal{Z},$
where $ \langle \cdot,  \cdot\rangle_\mathcal{Z}$ denotes the inner product in $\mathcal Z$. The function $\varphi$ is defined implicitly by the choice of the kernel.  

We shall use the one-stage large-scale optimization model \eqref{one-stage-obj}--\eqref{one-stage-Znon-neg}. To reflect that the data is mapped into the space $\mathcal{Z}$, we need modify constraints \eqref{one-stage-ziell} as follows
\begin{equation}
\label{one-stage-ziell-mod}
z^i_\ell \geq \psi_j\big(\varphi(x^i_\ell)\big) - \psi_i \big(\varphi(x^i_\ell)\big) +1 \quad 
     i \in\Jc, \;  \ell = 1, \ldots, m_i, \; \; j\in \Jc^{-i}
\end{equation}
Due to the finite number of realizations, we also can view $Q$ as a vector in $\Rb^N.$ 
Similarly, the random variables $Z^i$, $i\in\Jc$, can be viewed as vectors in $\Rb^{m_i}$. 
We observe that for the optimal solution $(\hat{v},\hat{Q},\hat{Z})$, the last two elements: $\hat{Q}$ and $\hat{Z}$, have at least one positive component unless the classes are separable ideally.  

Notice that problem \eqref{one-stage-obj}--\eqref{one-stage-Znon-neg} is a convex optimization problem that obviously satisfies The Slater's constraint qualification condition. We assign Lagrange multipliers $\lambda_i\geq 0$ to each constraint \eqref{one-stage-qi} and $\mu_{j,\ell}^i\geq 0$ to each constraints \eqref{one-stage-ziell-mod}. The Lagrange function has the form
\begin{equation} 
\begin{aligned}
L(v,Q,Z,\lambda, \mu) & = \varrho_{\rm o}[Q] + \sigma \left\|v\right\|^2 + \sum_{i\in\Jc} \lambda_i\big( \varrho_i[Z^i] -q_i\big) \\ 
& \quad +  \sum_{i\in\Jc} \sum_{j \in \Jc^{-i}} \sum_{\ell=1}^{m_i} \mu_{j, \ell}^i\big(\left\langle v^j, \varphi\left(x_\ell^i\right)\right\rangle-\left\langle v^i, \varphi\left(x_\ell^i\right)\right\rangle+1 -z_\ell^i\big).\\
& =  \varrho_{\rm o}[Q] - \sum_{i\in\Jc} \lambda_i q_i +  
    \sum_{i\in\Jc} \Big( \lambda_i\varrho_i[Z^i] - \sum_{j \in \Jc^{-i}} \sum_{\ell=1}^{m_i} \mu_{j, \ell}^i z^i_\ell\Big) \\
 & \quad +     \sum_{i\in\Jc} \Big( \sigma \|v^i\|^2 + \sum_{j \in \Jc^{-i}} \sum_{\ell=1}^{m_i} \mu_{j, \ell}^i\big(\left\langle v^j-v^i, \varphi\left(x_\ell^i\right)\right\rangle+1\big)\Big). 
\end{aligned}
\label{eq:lagrange-total}
\end{equation}
To calculate the dual function, we minimize the Lagrange function with respect to $(v,Q,Z)$.  Minimization with respect to $Q$ leads to the problem
\begin{equation}
\label{p:q_constr}
    \min_{Q} \varrho_{\rm o}[Q] - \sum_{i\in\Jc} \lambda_i q_i. 
\end{equation}
Hence,  $\lambda \in \partial \varrho_{\rm o}[\hat{Q}]$ is necessary for the minimum to be finite, where $\partial \varrho_{\rm o}[\hat{Q}]$ refers to the convex subdifferential of $\varrho_{\rm o}(\cdot)$ at $\hat{Q}$. Since $\varrho_{\rm o}(\cdot)$ is positively homogeneous, we obtain that the minimum is either zero or $-\infty$. 

Further, we denote $\mu_{j}^i\in\Rb^{m_i}$ the vector with components $\mu_{j,\ell}^i$, $\ell=1,\dots m_i$, and we consider 
\begin{equation}
\label{p:z_constr}
    \min_{Z^i\geq 0} \lambda_i\varrho_i(Z^i) - \big\langle Z^i, \sum_{j \in \Jc^{-i}} \mu_{j}^i \big\rangle. 
\end{equation} 
The optimality conditions for problem \eqref{p:z_constr} state that at the optimal solution $\hat{Z}^i$, a subgradient $\zeta^i\in\partial \varrho (\hat{Z}^i)$ exists such that 
\begin{equation}
\label{e:opt_cond_z}
  \lambda_i\zeta^i - \sum_{j \in \Jc^{-i}} \mu_{j}^i \geq 0 \quad\text{and}\quad \langle \hat{Z}^i , \lambda_i\zeta^i - \sum_{j \in \Jc^{-i}} \mu_{j}^i \rangle  =0.  
\end{equation}
Since $\varrho (\hat{Z}^i) = \langle \hat{Z}^i , \zeta^i \rangle$ for $\zeta^i\in\partial \varrho (\hat{Z}^i)$, we conclude that
\[
    \min_{Z^i\geq 0} \lambda_i\varrho (\hat{Z}^i) - \big\langle \hat{Z}^i, \sum_{j \in \Jc^{-i}} \mu_{j}^i \big\rangle =0. 
\]
Notice that if $\hat{Z}^i$ is such that for all $\zeta^i\in\partial \varrho (\hat{Z}^i)$ a component $\ell$ exists such that 
$\lambda_i\zeta^i_\ell - \sum_{j \in \Jc^{-i}} \mu_{j,\ell}^i < 0$, then problem \eqref{p:z_constr} is unbounded and, hence, the dual function is infinite.  Thus, $\lambda_i\zeta^i \geq \sum_{j \in \Jc^{-i}} \mu_{j}^i$ is required for $\mu$ to be in the domain of the dual function.

We denote $\tilde{\mu}^i = \sum_{j \in \Jc^{-i}} \mu_{j}^i$. 
Proposition 2 in \cite{AADD_nonlinear} implies that 
\[
        \partial \varrho_\sys(Z) = \Big\{ 
        \theta \in \Lc_q(\Omega, \Fc, P;\Rb^N) : \theta_i = \lambda_i \zeta_i : ~ \lambda \in \partial \varrho_{\rm o}[Q], ~ \zeta_i \in \partial\varrho_i (Z^i)\text{ for all } i\in\Omega_N.  
        \Big\}.
\]
Hence, we recognize that \eqref{e:opt_cond_z} states that
\begin{equation}
\label{e:1opt_cond_z}
  \tilde{\mu}^i \leq \theta^i \quad\text{and}\quad \hat{Z}^i_\ell (\theta^i_\ell  - \tilde{\mu}^i_\ell)  =0\; \text{ for all } i\in\Jc, \text{ and for all } \ell=1,\dots, m_i.  
\end{equation}
Turning to the minimization of the terms in the Lagrangian related to $v$, we obtain from the optimality conditions that for all $i\in\Jc$ the following equality holds:
\begin{equation}
\label{e:opt-conf-1st}
 2\sigma \hat{v}^i - \sum_{j \in \Jc^{-i}} \sum_{\ell=1}^{m_i} \mu_{j, \ell}^i\varphi(x_\ell^i)   
                   + \sum_{j \in \Jc^{-i}} \sum_{\ell=1}^{m_j} \mu_{i, \ell}^j\varphi(x_\ell^j) = 0.  
\end{equation}
Hence, for all $i\in\Jc$, we obtain the representation
\begin{equation}
\label{e:v-represent}
 \hat{v}^i = \frac{1}{2\sigma } \Big( \sum_{k \in \Jc^{-i}} \sum_{\ell=1}^{m_i} \mu_{k, \ell}^i\varphi(x_\ell^i)  
                   - \sum_{k \in \Jc^{-i}} \sum_{\ell=1}^{m_k} \mu_{i, \ell}^k\varphi(x_\ell^k)\Big).  
\end{equation}
It implies that $\|v\|^2$  can be calculated by using only the kernel $K(\cdot,\cdot)$ as follows
\begin{align*}
\|v^i\|^2 = \frac{1}{4\sigma^2} & \bigg(
\sum_{k,s \in \Jc^{-i}} \sum_{\ell=1}^{m_i} \sum_{r=1}^{m_i} \mu_{k,\ell}^i \mu_{s,r}^i\, K(x_\ell^i, x_{r}^i) - 2 \sum_{k,s \in \Jc^{-i}} \sum_{\ell=1}^{m_i} \sum_{r=1}^{m_{s}} \mu_{k,\ell}^i \mu_{i,r}^{s}\, K(x_\ell^i, x_{r}^{s}) \\
& + \sum_{k,s \in \Jc^{-i}} \sum_{\ell=1}^{m_k} \sum_{r=1}^{m_{s}} \mu_{i,\ell}^k \mu_{i,r}^{s}\, K(x_\ell^k, x_{r}^{s})
\bigg).
\end{align*}
We designate the quadratic form by $g_i(\mu)$, i.e., 
\[  
\|v^i\|^2 = \frac{1}{4\sigma^2} g_i(\mu).
\]
Furthermore, the functions $ \langle v^j - v^i, \varphi(x_r^i)\rangle $ take on the form 
\begin{align*}
 h_{jr}^i (\mu) & = \langle v^j - v^i,  \varphi(x_r^i)\rangle  
  = \frac{1}{2\sigma} \Big\langle  \sum_{k \in \Jc^{-j}} \sum_{\ell=1}^{m_j} \mu_{k, \ell}^j\varphi(x_\ell^j)  
                   - \sum_{k \in \Jc^{-j}} \sum_{\ell=1}^{m_k} \mu_{j, \ell}^k\varphi(x_\ell^k) \\
 &\qquad                 
                  - \sum_{k \in \Jc^{-i}} \sum_{\ell=1}^{m_i} \mu_{k, \ell}^i\varphi(x_\ell^i)  
                   + \sum_{k \in \Jc^{-i}} \sum_{\ell=1}^{m_k} \mu_{i, \ell}^k\varphi(x_\ell^k), \varphi(x_r^i)\Big\rangle \\
& = \frac{1}{2\sigma} \bigg[ \sum_{k \in \Jc^{-j}} \sum_{\ell=1}^{m_j} \mu_{k, \ell}^j \, K(x_\ell^j, x_r^i)
- \sum_{k \in \Jc^{-j}} \sum_{\ell=1}^{m_k} \mu_{j, \ell}^k \, K(x_\ell^k, x_r^i) \\
&\quad\quad
- \sum_{k \in \Jc^{-i}} \sum_{\ell=1}^{m_i} \mu_{k, \ell}^i \, K(x_\ell^i, x_r^i)
+ \sum_{k \in \Jc^{-i}} \sum_{\ell=1}^{m_k} \mu_{i, \ell}^k \, K(x_\ell^k, x_r^i)\bigg].
\end{align*}

Denote by $H^i_r (\mu) = \max\big\{0,\max_{j\in\Jc^{-i}} \big(h_{jr}^i (\mu)+1\big)\big\}.$ 

We gather all components 
$r=1,\dots, m_i$ in a vector function $H^i(\mu)$, which represents a random function with realizations $H^i_r (\mu)$. Further, we form the vector function $H(\mu)$ comprising all $H^i(\mu)$, $i\in\Jc$; $H(\mu)$ can be viewed as a random function as well.  Notice that the functions $H^i_r (\cdot)$  are convex because all functions $h_{jr}^i (\cdot)$ are linear. 

Using our observations we infer the following form of the dual function over its domain:
\begin{equation}
\begin{aligned}
D(\lambda, \mu) &  = \sum_{i\in\Jc} \Big( \sigma \|v^i\|^2 + \sum_{j \in \Jc^{-i}} \sum_{r=1}^{m_i} \mu_{j, r}^i\big(\left\langle v^j-v^i, \varphi\left(x_r^i\right)\right\rangle+1\big)\Big) \\
& = \sum_{i\in\Jc}  \Big(\frac{1}{4\sigma} g_i(\mu) + \sum_{j \in \Jc^{-i}} \sum_{r=1}^{m_i} \mu_{j, r}^i  \big(h_{jr}^i (\mu)+1\big)\Big).
\end{aligned}
\label{eq:dulfct-2st}
\end{equation}
The dual problem takes on the form
\begin{equation}
\label{p:dual-onestage}
\begin{aligned}
\max_{\lambda, \mu,\theta}\; & \sum_{i\in\Jc}  \Big(\frac{1}{4\sigma} g_i(\mu) + \sum_{j \in \Jc^{-i}} \sum_{r=1}^{m_i} \mu_{j, r}^i  \big(h_{jr}^i (\mu)+1\big)\Big)\\
\text{s.t. }\; & \sum_{j \in\Jc^{-i} } \mu_{j}^i \leq \theta_i, \quad \theta \in \partial \varrho_\sys (H(\mu)),\; i\in\Jc. \\
&  \mu_{j}^i \geq 0, \quad \forall i\in\Jc, \;j\in\Jc^{-i} .  
\end{aligned}  
\end{equation}
Notice that for the optimal solution $\hat{\mu}$, the components $\hat{\mu}_{j, \ell}^i =0$ for all $\ell =1,\dots, m_i$ such that 
$h_{j\ell}^i + 1 < 0$ and  $\sum_{j \in \Jc^{-i}} \hat{\mu}_{j,\ell}^i = \theta^i_\ell $ when $h_{j\ell}^i + 1 > 0$. 
The dual problem always has an optimal solution because the subdifferential $ \partial \varrho_\sys (\cdot)$ is a compact set, making the feasible set of problem \eqref{p:dual-onestage} compact as well. Additionally, for most of the coherent measures of risk analytical description of the subdifferential is available. 
We obtain an equivalent formulation of problem \eqref{p:dual-onestage} by using the dual representation of $\rho_\sys(\cdot)$. Denoting the empirical distribution function on $(1,\dots, m_i)$ by $P_i$, problem \eqref{p:dual-onestage} takes on the form:
\begin{equation}
\label{p:newdual-onestage}
\begin{aligned}
\max_{\lambda, \mu,\zeta}\; & \sum_{i\in\Jc}  \Big(\frac{1}{4\sigma} g_i(\mu) + \sum_{j \in \Jc^{-i}} \sum_{r=1}^{m_i} \mu_{j, r}^i  \big(h_{jr}^i (\mu)+1\big)\Big) + \sum_{i\in\Jc} \sum_{\ell =1}^{m_i} c_i\lambda_i\langle\zeta^i, H^i(\mu)\rangle \\
\text{s.t. }\; & \sum_{j \in\Jc^{-i} } \mu_{j}^i \leq \lambda_i\zeta^i, \quad \lambda\in \Ac_{o} (c),\;  \zeta^i \in \Ac_{\varrho_i}(P_i)\;\; i\in\Jc,\\
&  \mu_{j}^i \geq 0, \quad \forall i\in\Jc, \;j\in\Jc^{-i}\!\!.  
\end{aligned}  
\end{equation}
We recognize the structure of a two-stage problem, in which $\mu$ is a first-stage variable, while $(\lambda,\zeta)$ are second-stage variables. 
Hence, we arrive at the following statement.
\begin{theorem}
Assume that $K:\mathcal{Z}\times\mathcal{Z}\to\Rb$ is a positive-definite kernel function with representation
$K(x,x') = \langle \varphi (x),  \varphi (x')\rangle_\mathcal{Z}.$ Consider the classification problem \eqref{p:first-stage}--\eqref{p:ksecond-stage}, where 
\begin{equation}
\label{p:ksecond-stage}
\begin{aligned}
   R_i(\vartheta) = \min_{Z^i\geq 0} \quad & \varrho_i[Z^i]\quad
    \text{s.t.} \quad  z^i_\ell \geq \langle v^j, \varphi(x^i_\ell)\rangle - \langle v^i, \varphi(x^i_\ell)\rangle +1 \quad 
     j \in \Jc^{-i}, \;  \ell = 1, \ldots, m_i.
\end{aligned}
\end{equation}
The dual problem to \eqref{p:first-stage}--\eqref{p:ksecond-stage} is \eqref{p:newdual-onestage}. 
A new data $x$ is classified in class $\kappa$, $\kappa\in\Jc$ if the maximum 
\begin{equation}
\label{eq:kernel-classify}
\max_{i\in\Jc}  \Big( \sum_{k \in \Jc^{-i}} \sum_{\ell=1}^{m_i} \hat{\mu}_{k, \ell}^i K(x_\ell^i,x)  
                   - \sum_{k \in \Jc^{-i}} \sum_{\ell=1}^{m_k} \hat{\mu}_{i, \ell}^k K(x_\ell^k,x)\Big)
\end{equation}
 is achieved for $i=\kappa,$
 where $\hat{\mu}$ is the solution of problem \eqref{p:dual-onestage}.
\end{theorem}
\begin{proof}
Problem \eqref{p:first-stage}--\eqref{p:ksecond-stage} is equivalent to the one-stage problem 
\begin{align*}
\min_{v,\gamma,Q,Z} \quad & \varrho_{\rm o}[Q] + \sigma \sum_{i\in\Jc} \left\|v^i\right\|^2, \\
\text{s.t. }\; & 
   q_i \geq  \varrho_i(Z^i),\quad i\in\Jc,\\
     & z^i_\ell \geq \psi_j\big(\varphi(x^i_\ell)\big) - \psi_i \big(\varphi(x^i_\ell)\big) +1 \quad 
     i \in\Jc, \;  \ell = 1, \ldots, m_i, \; \; j\in \Jc^{-i},  \\
    & Z^i \geq 0 \quad i\in\Jc.  
\end{align*}
We have derived the dual problem to it and have shown that it is equivalent to problem \eqref{p:newdual-onestage}.
According to the classification rule, we assign a new data $x$ to class $\kappa$, $\kappa\in\Jc$ if the maximum 
of $\psi_\kappa\big(\varphi(x)\big)  = \max_{j\in \Jc\; j\neq i} \big(\varphi(x)\big)$. 
Using formula \eqref{e:v-represent}, we obtain the result.
    \end{proof}

\section{Numerical method for solving the two-stage problem}

 We note that our methods are described by individual risk measures for each of the classes but we believe that measuring classification errors with the same risk function facilitates fairness and might be more appropriate. 
To solve the two-stage problem \eqref{p:first-stage}-\eqref{p:second-stage}, we propose a specialized risk-averse multi-cut decomposition method. Our starting point is the risk-averse version of the multi-cut method in \cite{gulten2015two}.
The method is called "multi-cut" because a separate approximations of the individual risk measures are constructed by cutting planes. The regularization term controls the length of the steps during each iteration bringing stability to the process. Additionally, our regularization method has the property of converging over unbounded regions.

Note that we can view the set $\Ac_{\rm o}(c)$ as a subset of the unit simplex in $\Rb^{N}$ containing probability measures.  The dual representation for $\varrho_{\rm o}$ becomes
\begin{equation}
\varrho_{\rm o}[R(\vartheta)] = \max_{\mu\in\Ac_{\rm o}} \sum_{i\in\Jc}  \mu_i R_i(\vartheta_i).
\label{eq:obj-cut-source}
\end{equation}
Similarly, the dual sets $\mathcal{A}_i(P_i)$ associated with $\varrho_i$ become subsets of the unit simplex in $\Rb^{m_i}$ containing probability measures.  The dual set of many popular risk measures is described analytically; see \cite{DDARriskbook}. 

In the course of the proposed iterative method, the set $\Ac_{\rm o}$ is approximated by a finite subset of measures $\{\mu^\kc, \kc=1,\ldots,k\}$ that are collected at each iteration. The first-stage objective value is approximated by a variable $\alpha$ and the second-stage objective value of the $i$-th scenario is approximated by $r_i$.  According to \eqref{eq:obj-cut-source}, we construct the following cut (approximating tangent plane) at iteration $k$:
\begin{equation}
\alpha \geq \sum_{i\in\Jc} \mu_i^k r_i.
\label{eq:obj-cut}
\end{equation}
The subgradient $\mu_i^k$, $i\in\Jc$ at iteration $k$ are calculated by identifying those elements in  $\Ac_{\rm o}$ for which the maximum in \eqref{eq:obj-cut-source} is attained. 
The constraint \eqref{eq:obj-cut}, which may be called an objective cut, approximates the first-stage objective function from below.
For the second-stage problem of the $i$-th scenario, let $\pi_{j\ell}^i$ stand for the optimal Lagrange multipliers associated with the inequality constraints  
\[
-z^i_\ell \leq \langle \vartheta^{ik}, (x^i_\ell, -1) \rangle  -  \langle \vartheta^{jk}, (x^i_\ell, -1) \rangle + 1 \quad \ell=1,\dots m_i,\;j\in\Jc^{-i}. 
\]
where $\vartheta^{jk}$ for $j\in\Jc$ stands for the first-stage decision variables at the $k$-th iteration of the method. 
Further, let $D_i$ denote the expanded data matrix for class $i$:
\[
D_i=\begin{pmatrix}x^i_{11} & x^i_{12} & \ldots & x^i_{1n} & -1 \\ x^i_{21} & x^i_{22} & \ldots & x^i_{2n} & -1 \\ \vdots & \vdots & \ddots & \vdots & \vdots \\ x^i_{m_i1} & x^i_{m_i2} & \ldots & x^i_{m_in} & -1\end{pmatrix}.
\]
The subdifferential of the $i$-th second-stage optimal value function $\partial R_i(\vartheta^k)$ with respect to $\vartheta^k$ contains vectors $g^k$ whose components involve the optimal Lagrange multiplies $\pi_{j,\ell}^{ik}$ and are calculated as follows. For all $i\in\Jc$,
\begin{equation}
g_j^{ik} = 
\begin{cases} 
-D_i^\top \sum_{s \in \Jc^{-i}} \pi_{s}^{ik} & \text{if }  j=i.\\
D_i^\top \pi_{j}^{ik} & \text{otherwise.}
\end{cases}
\label{eq:g-formula}
\end{equation}
We construct the following regularized master problem at iteration $k$:
\begin{equation}
\begin{aligned}
\min_{\alpha, r, \vartheta} \ & \alpha + \sigma\|v\|^2+ \beta\|\vartheta-w^k\|^2\\\
\text{s.t.} \quad & \alpha \geq \sum_{i\in\Jc} \mu_i^\kc r_i, \quad \kc \in\Kc_0 \\
& r_i \geq R_i^\kc + \sum_{j\in\Jc} \langle g_{j}^{i\kc} , \vartheta^j - \vartheta^{j\kc}\rangle  \quad \kc\in \Kc_i,\;  \ i\in\Jc, \\
& \alpha, r_i \geq 0 \quad i\in\Jc.
\end{aligned}
\label{p:rd-rgmaster}
\end{equation} 
The parameter $\beta>0$ is the regularization parameter associated with the weight of the proximal term.
The solution $\alpha^k$ of problem \eqref{p:rd-rgmaster} provides an approximation of the systemic measure of risk in the first stage problem \eqref{p:first-stage}.  The new proximal center $w^k = (w_1^k,\ldots,w_N^k)$ is updated within each iteration based on the relation among 
$\varrho_{\rm o}[R(\vartheta^k)]$, $\varrho_{\rm o}[R(w^{k-1})]$, and $\alpha^k$.  

We propose the following regularized multi-cut method with a parameter $\delta\in (0,1)$.

\begin{minipage}{0.98\textwidth}
   \begin{tightitemize}
    \item[]   \vspace{0.2cm} \hrule \vspace{0.2cm}
                \textbf{Multi-class Classification with Systemic Risk } 
                \vspace{0.2cm} \hrule \vspace{0.2cm}  
\item[\emph{Step 0.}] Set $ k = 1 $. Choose initial decision variable $\vartheta^1$ with $\|v^i\|=1$ for all $i\in\Jc$. 
\item[\emph{Step 1.}] For each $i\in\Jc $, solve the second-stage problem \eqref{p:second-stage}. Let $R_i^k$ be its optimal value. Calculate the subgradients $g_i^k = (g_{i1}^k,\ldots, g_{iN}^k)$ by formula \eqref{eq:g-formula} and add the new cuts to $\Kc_i$, $i\in\Jc$.
\item[\emph{Step 2.}] Calculate the systemic risk $\varrho^k =\varrho_{\rm o}[R^k]$ where $ R^k$ has realizations $R_i^k\; i\in\Jc$ and 
calculate $\mu^k$ at point $R^k$ using \eqref{eq:obj-cut-source} and include the new cut in the set $\Kc_0$. 
\item[\emph{Step 3.}] Determine the new center $w^k$ as follows. If $k=1$ or 
\[
\varrho^k\leq (1-\delta)\bar{\varrho}^{k-1}+\delta \alpha^{k-1}, 
\]
then set $w^k=\vartheta^k$ and  $\bar{\varrho}^{k}=\varrho^k$ (descent step). Otherwise, set $w^k=w^{k-1}$ and  $\bar{\varrho}^{k}=\bar{\varrho}^{k-1}$   (null step). 
\item[\emph{Step 4.}] Solve the master problem \eqref{p:rd-rgmaster}. Denote the solution by $\alpha^k, \vartheta^k, r^k$.
\item[\emph{Step 5.}] If $\bar{\varrho}^{k} = \alpha^k $, then stop ($w^k$ is an optimal solution); otherwise continue.
\item[\emph{Step 6.}] Remove from the sets $\Kc_i$, $i=0, 1,\dots, N$, the constraints whose Lagrange multipliers at the solution of  \eqref{p:rd-rgmaster} are~0.
Increase $k $ by 1 and go to {Step 1}.
  \item[]   \vspace{0.2cm} \hrule \vspace{0.2cm}
\end{tightitemize}
\end{minipage}
Now, we show the convergence of the method.  
\begin{theorem}
The regularized multi-cut method generates a sequence
$\{w^k\}$ which converges to an optimal solution $\vartheta^*$ of problem \eqref{p:first-stage}--\eqref{p:second-stage}. 
Furthermore, the optimal solutions $\{\alpha^k\}$ of the master problem \eqref{p:rd-rgmaster} converge to the optimal classification risk value
$\varrho^*_{\rm sys}=\varrho_{\rm o}[R(\vartheta^*)]$ and $\lim_{k\rightarrow \infty}\|\vartheta^{k+1}-w^k\|  = 0.$
\end{theorem}
\begin{proof}
First, we observe that the second stage problem \eqref{p:second-stage} is always solvable and, hence, the two-stage problem has a complete recourse.
Therefore, the domain of the functions $R_i(\vartheta)$ is the entire space $\Rb^{N(n+1)}$ and the random variable $R(\vartheta)$ is bounded, which entails that the systemic risk is always finite. The regularization term in the objective function of the first-stage problem \eqref{p:first-stage}, makes the objective coercive. Hence, the first stage problem is also solvable and the entire two-stage classification problem has an optimal solution.
Due to the proximal term, the objective function of problem \eqref{p:rd-rgmaster} has compact level sets. Hence, the master problem is always solvable. After solving the second stage problem for all classes, we obtain all subgradients $g_i^k,$ $i\in\Jc$ necessary for the cutting plane approximation of the respective functions $R_i(\cdot)$. The approximation of the systemic risk measure by cutting planes uses the separate approximations of $R_i(\cdot)$. Therefore, the convergence of our method follows by the convergence properties of the regularized decomposition method for two-stage problems described in \cite{ruszczynski1986regularized}. We only need argue that the assumptions made there are satisfied: existence of an optimal solution and the existence of an uniform bound $C$ such that 
$\|g_{ij}^k\| \le C$ for all $k=1,2,\dots$ and all $i,j\in\Jc$. Due to the complete recourse and the fact that the risk measure is the support function of a compact closed set (by its dual representation), we conclude that an optimal solution the two-stage problem exists. The only situation, in which the uniform boundedness of the vectors $g_{ij}^k$ may be violated is when $\vartheta^k$ is a boundary point of the domain of $R(\cdot)$. Since the domain is the entire space, we infer that this assumption is satisfied as well.
\end{proof}

\begin{theorem}
\label{t:complexity}
The computational complexity of the numerical method for classification with systemic risk of $N$ classes based on data set of $n-$ dimensional points increases linearly with the number of data points.  
\end{theorem}
\begin{proof}
First we show that all relevant information from the second-stage problem is obtained in a closed form. 
First, the class–wise second–stage problem admits a closed form solution and  the optimal Lagrange multipliers, which are needed in the method, can also  be written in closed form.  As a result, the second stage problem can be handled without calling a solver. 
We provide the form of these quantities. Define the margins 
$b_{j\ell}^i(\vartheta) = \langle \vartheta_j,x^i_\ell\rangle-\langle \vartheta_i,x^i_\ell\rangle+1$, $~j\in\Jc^{-i}$, $\ell=1,\dots,m_i.$
The second-stage problem \eqref{p:second-stage} for any $i\in\Jc$ takes on the form:
\begin{align}
\min_{Z^i} \quad & \varrho_i[Z^i]\\
    \text{s.t.} \quad  
& z^i_\ell \;\ge\; b_{j\ell}^i(\vartheta), \quad j\in\Jc^{-i},\; \ell=1,\dots,m_i,\label{zi-constr}\\
& Z^i\ge 0.
\end{align}
The value of $b_{j\ell}(\vartheta)$ are known once the first stage decision variables are given. 
We obtain the solution $\hat{z}_\ell^i =\max\big(0, \max_{j\in\Jc^{-i}} b_{j\ell}(\vartheta)\big).$
Let $I_\ell\subset \Jc^{-i}$ contain all indeces $j\in\Jc^{-i}$ for which $z^i_\ell = b_{j\ell}^i(\vartheta)$. 
Assigning dual variables $\pi_{j\ell}^i\ge 0$ to constraints \eqref{zi-constr}, 
the optimality conditions state that at the optimal solution $\hat{Z}^i$, 
a subgradient $\zeta^i\in\partial \varrho (\hat{Z}^i)$ exists such that 
\begin{equation}
\label{e:opt_cond_compl}
  \zeta^i - \sum_{j \in \Jc^{-i}} \pi_{j}^i \geq 0 \quad\text{and}\quad \langle \hat{Z}^i , \zeta^i - \sum_{j \in \Jc^{-i}} \pi_{j}^i \rangle  =0.  
\end{equation}
Let $\xi^i$ be any subgradient in $\partial \varrho_i (\hat{Z}^i)$. 
Recall that $\varrho_i(\hat{Z}^i) = \langle \hat{Z}^i , \xi^i \rangle$ for any $\xi^i\in\partial \varrho (\hat{Z}^i)$ and that $\xi^i\geq 0$.
The complementarity conditions imply, that the optimal Lagrange multipliers are 
\begin{gather*}
    \hat{\pi}^i_{j\ell} = \begin{cases}
0, & \text{ if }  \hat{z}_\ell^i = 0 > b_{j\ell}^i(\vartheta),\\
\text{any non-negative value:  }  \xi^i_\ell \geq  \sum_{j \in I_\ell} \pi_{j\ell}^i & \text{ if } \hat{z}_\ell^i = b_{j\ell}^i(\vartheta) = 0,\\
\text{any non-negative value:  } \xi^i_\ell = \sum_{j \in I_\ell} \pi_{j\ell}^i & \text{ if } \hat{z}_\ell^i = b_{j\ell}^i(\vartheta) > 0.
\end{cases}
\end{gather*}
Since the components of $\pi_j^i$ do not interact, the distribution of the value of the respective component of $\xi^i$ is always possible. Therefore, we can use any subgradient of $\varrho(\hat{Z}^i)$ at the calculated optimal solution. Notice also, that for most of the components $\ell=1,\dots m_i$, the set $I_\ell$ will be a singleton, which will simplify the assignment. Hence, the optimality conditions \eqref{e:opt_cond_compl} can be satisfied by using any $\xi$. 
Therefore, the computational effort of this part will increase linearly with the number of data points in the training set. 

Recall that $N=|\Jc|$ represents the number of classes, $n$ denotes the number of features, and denote the overall sample size is $M = \sum_{i\in\Jc} m_i$. 
The regularized decomposition method solves at each iteration only the master problem, which has $1 + N(n{+}2)$ variables. All constraints are linear. Hence, from the optimality conditions follows that at most $1 + N(n{+}2)$ Lagrange Multipliers may be positive. Thus, after Step 6, at most $1 + N(n{+}2)$
constraints are kept in \eqref{p:rd-rgmaster}. In Step 1, at most $N$ new inequalities are generated. Therefore, the problem \eqref{p:rd-rgmaster} that is solved in Step 4, has no more than $1 + N(n{+}3)$ constraints. We conclude that the regularized decomposition master is independent of the number of points in the data set. 
\end{proof}
We note that the size of the one-stage problem formulation \eqref{one-stage-obj}-\eqref{one-stage-Znon-neg} depends on the sample size: it has $N(n{+}2) + M$ variables and $N(M+1)$ constraints.  

\section{Risk-averse classification and fairness} 
\label{sec:fair}

The issue of fairness has attracted significant attention in the machine learning community recently, highlighting concerns about biased outcomes in automated decision-making. These algorithms, used in areas such as job recruitment, credit risk assessment and others, might carry on or even aggravate social biases if not carefully designed. As a result, researchers and practitioners are increasingly focused on ensuring these technologies being both accurate and fair.
Current work regarding fairness machine learning presents three different approaches: pre-processing the data before learning \cite{feldman2015certifying, kamiran2012data, zemel2013learning}, forcing fairness through the learning process or within the learning model \cite{rychener2022metrizing, roh2020fairbatch, donini2018empirical, fish2016confidence}, and post-processing the results after learning \cite{chzhen2020fair, pleiss2017fairness}. 

To evaluate the fairness of the results and compare among different methods, it is crucial to quantify fairness using some metrics. Previous research works on fairness introduce several different metrics. We adopt the commonly used notion of \emph{Equal Opportunity (EO)}. The definition that is suggested in \cite{wang2024wasserstein} requires the comparison of two probabilities as follows.
     Given a labelled dataset $D = (X_i, L_i, Y_i)_{i \in \Jc}$, where $X_i$ represents the vector of features, $L_i \in \{0,1\}$ represents the label, and $Y_i \in \{0,1\}$ represents a sensitive attribute, a classifier $\psi$ satisfies \emph{equal opportunity} if
\[
 P[\psi(X_i) = 1 \mid L_i=1,\, Y_i = 0] = P[\psi(X_i) = 1 \mid L_i=1,\, Y_i = 1].
\]
Here, the probability $P$ is the frequency within the given dataset used as an estimation of the true conditional probability.
In other words, the estimated probability of a positive classification among truly positive examples ($L_i=1$) is the same across the values of the sensitive attribute $Y$ that indicate groups $0,\,1$.  The difference between these true positive rates quantifies the fairness of a given classifier. Instead of the difference, it is common to consider the ratio of the two probabilities. Therefore, we adopt the following fairness metric, which we call the \emph{EO-ratio (EOr)}:
\begin{definition}
Given a labeled dataset $D = (X_i, L_i, Y_i)_{i \in \Jc}$ with the same notation as above, a classifier $\psi$ has an \emph{EO-ratio $\tau$}, where
\[
\tau = \min \left( \frac{P[\psi(X_i) = 1 \mid L_i=1,\, Y_i = 1]}{P[\psi(X_i) = 1 \mid L_i=1,\, Y_i = 0]}, \ \frac{P[\psi(X_i) = 1 \mid L_i=1,\, Y_i = 0]}{P[\psi(X_i) = 1 \mid L_i=1,\, Y_i = 1]} \right)
\]
with $P$ being the frequency within the given dataset.
\end{definition}
The closer $\tau$ is to 1, the fairer is the classifier $\psi$ with respect to the groups $0,1$. Again, the probabilities here can be numerically estimated by the frequencies of the predictions in each group restricted to $L_i=1$. These two definitions are very close to each other  

We propose a risk-averse formulation using contextual risk measures and the mean-semideviation as the outer risk measure. Recall the formulation \eqref{p:first-stage}, the objective has two terms: an expected value of the misclassification risk of each class, and an expected value of the shortfall of the risk of each class and the average risk. The second term, the shortfall, can be viewed as a penalty term with parameter $c\in[0,1]$,  penalizing every class whose risk is higher than the average risk. We can call it a fairness term since it prevents any class from being overlooked and having a risk too much higher than other classes.

The models that we propose is the following. We consider the scenario in which a fair-sensitive categorical feature with values in $\mathcal{Y}={1,\dots, S}$ is present in the data points to be classified. To make the prediction fair with respect to the $S$ groups within all classes, we introduce a contextual risk-measure $\varrho_c[Z]$. Assume that the random classification error is given by $Z=f(\vartheta,X,Y)$, where $\vartheta$ determines the classifier, $X$ stands for a random data point and $Y$ stands for the value of said feature. Let $f$ be measurable, convex and monotonically non-decreasing with respect to the second argument for any value of the first and the third argument. We need consider the classification error $Z^i$ in every context $y\in\mathcal{Y}$ for all classes $i\in\Jc.$
Consider the probability space $(\Omega, \Fc,P)$ where the random vector $Z$ lives and define the spaces $\mathcal{Y}_s=(\Omega,\Fc|Y=y,P(X|Y=y)).$  A contextual risk measure is a composition of coherent measures of risk evaluating the risk given a context and an aggregation measure providing the total risk for all contexts. 
More precisely, we fix risk measures $\rho_{i,s} [Z^i|Y=s]$ for each $i\in\Jc$ and $y\in\mathcal{Y}$ and choose 
$\varrho_c[Z^i]$ as an aggregation measure to obtain the total risk of each class over all contexts. Additional aggregation by an outer measure $\varrho_{\rm o}[Z]$ will aggregate the risk over the classes. 
\begin{proposition}
  Let  $\rho_{i,y} (\cdot)$ be coherent measures of risk for  all  $i\in\Jc$ and all $y\in\mathcal{Y}$, 
risk measures $\varrho_{i}: \Lc_\infty(\mathcal{Y},\Fc_\Yc, P_{i|Y})\to\Rb$ be coherent measures of risk for all $i\in\Jc$, and 
$\varrho_{\rm o}:\Lc_\infty(\Omega_N,\Fc_N, p)\to\Rb$ be coherent as well. 
Denote
\begin{equation*}
	    V_c(i,y) = \rho_{i,y}[Z^i|Y=y] \quad \text{and} \quad W(i) = \varrho_{i}[V_c(i,\cdot)],\quad y\in\mathcal{Y},\; i\in\Jc.
\end{equation*}
 Then the risk measure $\varrho_{\rm sys} [Z]= \varrho_{\rm o} [W]$ satisfies axioms A1-A4 of systemic measure of risk. 
\end{proposition}
\begin{proof}
  (i) Given any $Z_1$, $Z_2$ and $\alpha \in (0,1)$, we consider the random vector $Z_3 = \alpha Z_1 + (1-\alpha)Z_2$.  It follows that
\[
V_c^3(i,y) = \rho_{i,y}[Z^i_3|Y=y] \leq \alpha\rho_{i,y}[Z^i_1|Y=y] + (1-\alpha)\rho_{i,y}[Z^i_2|Y=y] = \alpha V_c^1(i,y)+ (1-\alpha) V_c^2(i,y).
\]
by the convexity of $\rho_{i,y}[\cdot]$ for all $i \in\Jc$ and all $y\in\mathcal{Y}$.  Hence
\[
W^3(i) =  \varrho_{i}[V_c^3 (i,\cdot)] \leq \alpha \varrho_i[V_c^1(i,\cdot)]+ (1-\alpha) \varrho_i[V_c^2(i,\cdot)] =
\alpha W^1(i)+ (1-\alpha) W^2(i).
\]
by the same arguments. Analogously,  
\[
\varrho_{\rm sys}[Z_3]= \varrho_{\rm o} [W^3]\leq \alpha \varrho_{\rm o} [W^1]+ (1-\alpha) \varrho_{\rm o} [W^2] =
\alpha \varrho_{\rm sys}[Z_1] +(1-\alpha) \varrho_{\rm sys}[Z_2], 
\] 
which establishes the convexity property. 

(ii) Suppose the vectors $Z_1, Z_2 $ satisfy $Z_1\leq Z_2$ a.s. This implies that $Z_1^i \leq Z_2^i$ a.s. and for all $y\in \mathcal{Y}$. Using the monotonicity of the risk measures $\rho_{i,y}[\cdot]$, $\varrho_i[\cdot]$ and $\varrho_{\rm o}[\cdot]$, we infer the monotonicity of $\varrho_{\rm sys}[\cdot].$  

 (iii) The positive homogeneity follows in a straightforward manner from the definition. 

	    (iv) Given a random vector $Z$ and a real constant $a$, we calculate the risk $V^+_c(i,y)$ of the translated error in context $y$  as follows:
     \[
     V^+_c(i,y) = \rho_{i,y}[Z^i +a|Y=y] = \rho_{i,y}[Z^i|Y=y] +a = V_c(i,y) + a \quad \text{for all } y\in\mathcal{Y}, \; i\in\Jc.
     \]
    The second equality holds by the translation property of the coherent measures of risk. This implies that
   $W^+(i) =  \varrho_{i}[V_c^+ (i,\cdot)] = \varrho_{i}[V_c (i,\cdot) + a] = W(i) +a $ by the same argument. Hence,  
	    $\varrho_{\rm o}[W+a]= \varrho_{\rm o}[W] + a,$ concluding that property (A4) holds as well.
\end{proof}  

This structure of measuring risk allows us to enforce fairness within each class by choosing the aggregator measures $\varrho_i$ to be such that deviation from the average risk or excessive risk above a certain quantile is penalized.

Consider as an example the case of two classes $A$ and $B$, both with a categorical feature $Y$ with values $\mathcal {Y} = \{1,2\}$.  
We could calculate the risk of the groups $ \varrho_{A_i}$ and $\varrho_{B_i}$, $i=1,2$, which would be the measures $\rho_{i,y} $ with $i=A,B$ and  $y=1,2.$ In the special case of the upper mean-semideviation of order 1 for the $\varrho_i$, we obtain the following form:
\begin{gather*}
    \varrho_A = \sum_{i=1,2} p_{A_i}\varrho_{A_i} + \kappa p_{A_1}p_{A_2} |\varrho_{A_1} - \varrho_{A_2}|,\\
    \varrho_B = \sum_{i=1,2} p_{B_i}\varrho_{B_i} + \kappa p_{B_1}p_{B_2} |\varrho_{B_1} - \varrho_{B_2}|,
\end{gather*}
where $\kappa\in (0,1)$ and $p_{A_i}$ is the (estimated) conditional probability of a data point to be in class $A$ with feature $Y=i,$ $i=1,2.$
In that case, we may consider a risk-neutral evaluation for the outer risk measure $\varrho_{\rm o}$, i.e. 
$\varrho_{\rm o} = p_A\varrho_A + p_B\varrho_B. $ We shall name classifiers like this \textbf{CNACR} meaning classification with nonlinear aggregation of contextual risks.

We note that the upper mean-semideviation of any order would penalize excessive risk above the average and it is suitable as a class aggregator measure enforcing fairness. 
The measures represented in \eqref{e:avar-comb} penalize excessive risk above a the quantile $\alpha$-quantile of the risk distribution and are also suitable as class aggregator.  Additionally, any convex combinations or maximum of those measures of risk will facilitate used as class aggregators will facilitate the design of a fair and robust classifier.

We still can use our numerical method for this model;  we only need change the way we calculate the subgradients and we can apply the regularized decomposition method to solve the modified fairness classification model. We test our ideas via numerical experiments and report them in the next section.

To put our proposal in the context of the existing approaches, we note that
our formulation forces fairness between the classes within the learning model. In this category, we have seen models that directly modify the loss function, such as \cite{rychener2022metrizing}, which uses a penalty term on the distance between the distributions of predictions within the two groups. Other works use extra constraints to force fairness, e.g. \cite{donini2018empirical}. Our two-stage model naturally involves a 'penalty' term within the formulation as it penalizes the deviation of the risk for each class from the average risk. in this way there is some similarity to \cite{rychener2022metrizing}. However, the application of the two-stage model requires a different training process.
We argue that this is a very attractive property that can force fairness in classification. 
In a way, our approach also involves preprocessing the data before learning because, we split each of the original classes into multiple classes and then learn the data by training our two-stage model. 

\section{Numerical Experiments}

In this section, we demonstrate the performance of the proposed models and methods with respect to robustness and fairness. In our first set of experiments, we use the well-known MNIST dataset to illustrate how the risk measures can affect the behavior of a classification model on noisy and mislabeled data sets. The second set of experiments show the viability of the developed risk-averse kernel formulation based on a data set that is known to be non-linearly separable for detecting Electrical Faults, classification data from \cite{ElectricalFaultDetection2024}. Our third set of experiments uses the drug usage dataset \cite{Dua2019} to present the potential of the proposed risk-averse classification method to enforce fairness. 

For the first set of experiments, we try to gain a better insight as to when the risk-averse methods provide more robustness. To this end, we devise several perturbations and apply a few different kinds of noises to the training data. We emphasize that the test data are left untouched.
Under every kind of perturbation, we compare the risk-neutral method \eqref{eq:cs-rn-obj}--\eqref{eq:cs-rn-nonneg}, which we use as a benchmark and the risk-averse method \eqref{eq:cs-ra}.  For the first few experiments, we only use the samples from 3 out of 10 classes to maintain a low-risk situation. We run through all possible selections and compare the two methods under the same setting to ensure that the comparisons are thorough and do not depend on chance.  We also conduct experiments using all 10 classes and we observe that the performance of the risk-averse method gets even better with a larger number of classes. While a mathematical proof of this phenomenon is a subject of further research, one intuitive reason might be that downward bias might not be present or less pronounced when optimizing coherent measures of risk. It is known that the Sample Average Approximation (SAA) has a downward bias, which in our context results in underestimating the real expected classification error. However, the minimized empirical estimators of some risk measures do not necessarily have a downward bias. Analysis of the bias in data-driven risk-averse optimization problems is still an open research question. 

Note, we only use up to 1000 samples from each class. This strengthens our arguments because we observe better results with the risk-averse model, which demonstrates that the risk-averse method is a better option when only limited amount of data is available or acquiring of more data is expensive.

For the risk-averse method, we adopt the upper mean-semi-deviation measure in the second-stage and the expectation as an aggregator in the first stage problem. 

Two sets of parameters need to be defined ion that problem.: the probability mass function $c$ and the weight $\kappa_i$ of the semi-deviation in the calculation of the risk measures.  For simplicity, we assume that $c$ is the uniform measure on $\Omega_N$ and $\kappa_i$ of all the classes are set as the same number between 0 and 1. These parameters can be better adjusted if needed and warranted in each specific problem.  The MNIST dataset is very balanced, which requires less adjustment of these parameters. While we do agree that navigating through different selections of these parameters will potentially create better results, our main focus here is to compare the robustness that the risk measures can provide.

\subsection{Mislabeled data}
\label{sub:mislabeled_data}

We introduce mislabeling into the training sets by randomly selecting $10\%$ data from each class, and then randomly inserting them into other classes without changing the original size of any class. We use 3  out of the 10 classes from the MNIST dataset and 1000 samples from each class in each experiment reported here.  The risk-averse and the benchmark model are run 100 times. Each time, we divide the selected data into training set and test set with a 0.3 test rate, and we use the same training set and test set for both models for a fair comparison. We report the F1 score of each class and the average F1 score of each run and then compare their distribution. We also calculate and report the risk statistics during training and testing. 

The parameter $\kappa\in[0,1]$ of the mean-semideviation determines how risk-averse the measure is; we call $\kappa$ the risk level. We make experiments with different values of the risk level but all classes use the same risk level. 

\begin{figure}[h!]
\centering
\includegraphics[width=0.95\linewidth]{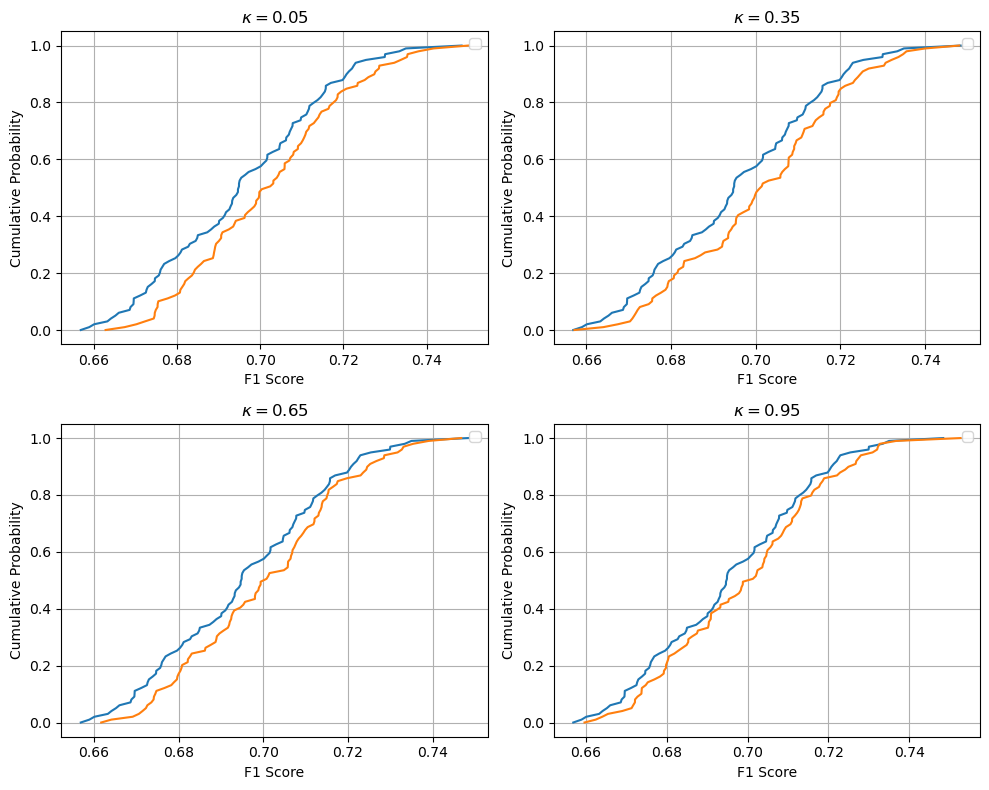}
\caption{CDF of the average F1 scores comparing the risk-neutral baseline (blue) and the risk-averse method (orange) with different risk levels $\kappa$ using mislabeled data.}
\label{fig:f1-mislabel}
\end{figure}

In Fig.\ref{fig:f1-mislabel}, we show the results using 4 different risk levels.  Recall that probability distributions can be compared by using stochastic orders. The first order dominance compares the associated distribution functions
A random variable with cumulative distributions function $F$ is \emph{stochastically larger} than a random variable with distribution function $G$  with respect to the first-order stochastic dominance if $F(\eta) \leq G(\eta)$  for all $\eta\in\Rb$. 

The plots in Fig.\ref{fig:f1-mislabel} show that with any selection of the risk level, the risk-averse method results in significantly higher F1 score than the original Crammer-Singer method does because we observe almost first-order stochastic dominance between the respective distribution functions. Moreover, we see that the choice of risk level plays a substantial role, as a smaller risk level clearly provides better results. However, this risk level cannot be too close to 0, as we have also tested. A risk level of 0.01 would provide a result with no significant difference from the risk-neutral model.

Though not reported here, we have also inspected the F1 score for each class during each run, and we observe CDF plots with a similar pattern as shown in Fig.\ref{fig:f1-mislabel}. This means that the introduction of risk measures makes the model significantly more immune to the risky environment created by mislabeled data. Additionally, we have conducted a pairwise t-test on the average F1 score between the baseline method and the risk-averse method testing the hypothesis that the F1 scores are the same. Even for the worst situation in the Fig.\ref{fig:f1-mislabel} with $\kappa=0.95$, we still reject this hypothesis with a $-6.16$ t-statistic and a $1.54\times 10^{-8}$ p-value.  

\begin{table}[h!]
  \centering
  \small
  \begin{tabular}{clcccccc}
    \toprule
    & Error& \multicolumn{3}{c}{Risk-averse} & \multicolumn{3}{c}{Baseline} \\
    \cmidrule(r){3-5} \cmidrule(l){6-8}
    & statistics & Class 1 & Class 2 & Class 3 & Class 1 & Class 2 & Class 3 \\
    \midrule
    \multirow{4}{*}{$\kappa = 0.05$}
    & Train Exp Val        & 0.221881 & 0.129303 & 0.246989 & 0.218315 & 0.121311 & 0.240600 \\
    & Train MSD ($\kappa=1$) & 0.417047 & 0.240529 & 0.461556 & 0.418637 & 0.232943 & 0.459031 \\
    & Test Exp Val         & 1.924157 & 0.935457 & 2.922738 & 2.281225 & 1.098725 & 3.560530 \\
    & Test MSD ($\kappa=1$)  & 3.238615 & 1.598175 & 4.833257 & 3.848408 & 1.886506 & 5.910646 \\
    \midrule
    \multirow{4}{*}{$\kappa = 0.35$}
    & Train Exp Val        & 0.222637 & 0.128429 & 0.247221 & 0.218315 & 0.121311 & 0.240600 \\
    & Train MSD ($\kappa=1$) & 0.424351 & 0.244557 & 0.468724 & 0.418637 & 0.232943 & 0.459031 \\
    & Test Exp Val         & 1.896190 & 0.919227 & 2.863517 & 2.281225 & 1.098725 & 3.560530 \\
    & Test MSD ($\kappa=1$)  & 3.193151 & 1.570841 & 4.732428 & 3.848408 & 1.886506 & 5.910646 \\
    \midrule
    \multirow{4}{*}{$\kappa = 0.65$}
    & Train Exp Val        & 0.222310 & 0.132216 & 0.247531 & 0.218315 & 0.121311 & 0.240600 \\
    & Train MSD ($\kappa=1$) & 0.412122 & 0.240425 & 0.456226 & 0.418637 & 0.232943 & 0.459031 \\
    & Test Exp Val         & 1.946301 & 0.947653 & 2.972493 & 2.281225 & 1.098725 & 3.560530 \\
    & Test MSD ($\kappa=1$)  & 3.274515 & 1.617666 & 4.919666 & 3.848408 & 1.886506 & 5.910646 \\
    \midrule
    \multirow{4}{*}{$\kappa = 0.95$}
    & Train Exp Val        & 0.223581 & 0.136028 & 0.248951 & 0.218315 & 0.121311 & 0.240600 \\
    & Train MSD ($\kappa=1$) & 0.408529 & 0.242401 & 0.452615 & 0.418637 & 0.232943 & 0.459031 \\
    & Test Exp Val         & 1.963673 & 0.959721 & 3.015177 & 2.281225 & 1.098725 & 3.560530 \\
    & Test MSD ($\kappa=1$)  & 3.303290 & 1.637332 & 4.989961 & 3.848408 & 1.886506 & 5.910646 \\
    \bottomrule
  \end{tabular}
  \caption{Risk table of the experiments with mislabeled data. For comparison, we calculate the risk measures using these values with the same risk levels, $\kappa=0$ for the expectation and $\kappa=1$ for the mean-semideviation (MSD). }
  \label{tab:risk-mislabel}
\end{table}

We also calculate and report the two statistics for each model on the training and the test sets, which we report in Table.\ref{tab:risk-mislabel}.  We obtain the optimal classifier for each class and calculate the risk measures for the training and the test data. For better comparison, we calculate the values based on the same risk levels: $\kappa=0$ and $\kappa=1$. 
We observe that the baseline method provides smaller values in the training risk, yet the risk values on the test data are significantly larger than the results from the risk-averse method. This indicates that, once trained, \emph{the risk-averse method has better generalization toward unknown data.}

We point out that, even though the experiments are conducted on only 3 classes from the MNIST dataset, we have run the same experiments on all possible selection of 3 classes to make sure that the result does not occur by chance. All results can be reproduced, and one can easily see that the results show a similar pattern as we have reported.

\subsection{Data with removed features}

In our second experiment, we consider the situation when the amount of information is very limited. To this end, we only use 150 samples to run the experiment with a test rate of $0.5$.
We have conducted the experiments on 1000 samples from 6 classes. For each image, we randomly make 90\% of the pixels unavailable. While this proportion appears large, we observed that the removal of features actually does not have a huge influence on the classification result. Even with half of the features removed, both methods can still achieve an average 0.9 F1 score in almost every class, which means that the situation is not really of high risk given the number of observations. When the risk is not high enough, there will be no significant difference between the risk-averse method and the risk-neutral one.

\begin{figure}[h!]
\centering
\includegraphics[width=0.95\linewidth]{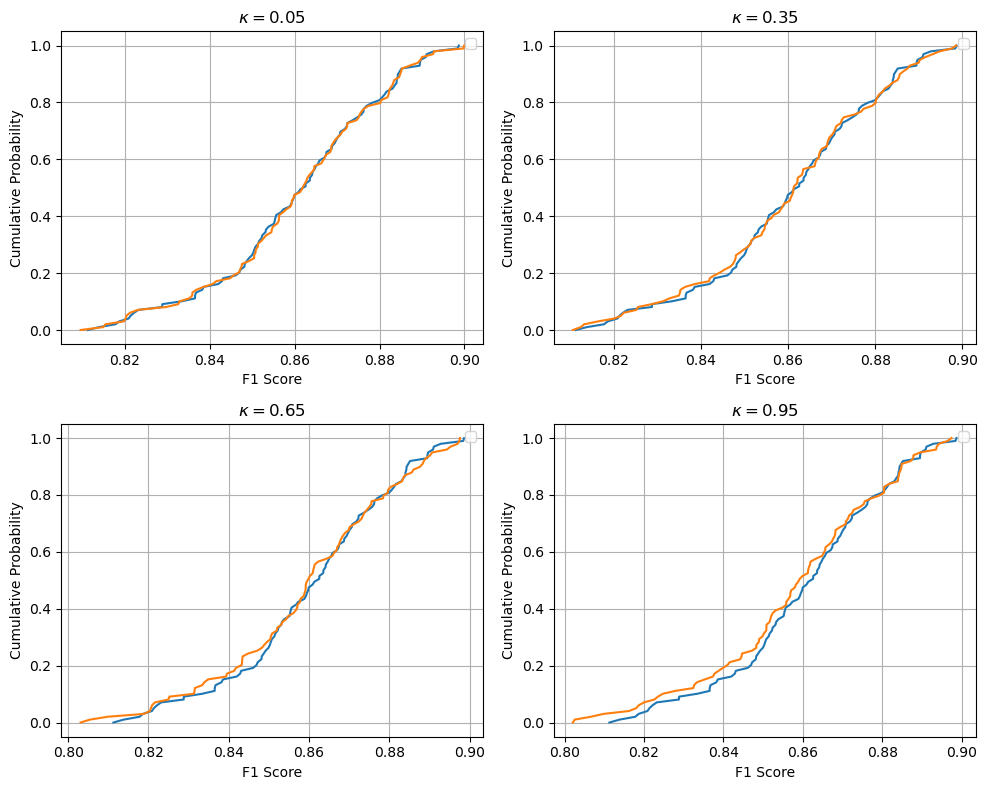}
\caption{CDF of the average F1 scores comparison of the risk-neutral baseline (blue) and the risk-averse method (orange) with different risk levels $\kappa$ using feature-removed data.}
\label{fig:f1-remove}
\end{figure}

According to Fig.\ref{fig:f1-remove}, we notice that when the parameter $\kappa$ has high value, the risk-neutral method actually outperforms the risk-averse one. This suggests that too high risk-aversion  has a negative effect. We can indeed observe that when the risk level decreases to 0.35, then a part of the risk-averse model's CDF falls significantly beneath the baseline's CDF, even though the t-test is still showing that the risk-neutral method might be better. When $\kappa=0.05$, the t-test gives the result with a $-5$ t-statistics and a $2.44\times 10^{-6}$ p-value, indicating that the risk-averse method is significantly better than the risk-neutral method concerning the average F1 score.

\begin{table}[h!]
  \centering
  \small
  \begin{tabular}{clcccccc}
    \toprule
    & Error & \multicolumn{6}{c}{Class} \\
    \cmidrule(r){3-8}
    & Statistics & 1 & 2 & 3 & 4 & 5 & 6 \\
    \midrule
    \multirow{8}{*}{$\kappa = 0.05$} 
    & Train Exp Val & \begin{tabular}{@{}c@{}}0.1522 \\ \raisebox{1.5ex}{\tiny$\big(0.1536\big)$}\end{tabular} & \begin{tabular}{@{}c@{}}0.1165 \\ \raisebox{1.5ex}{\tiny$\big(0.1170\big)$}\end{tabular} & \begin{tabular}{@{}c@{}}0.3497 \\ \raisebox{1.5ex}{\tiny$\big(0.3496\big)$}\end{tabular} & \begin{tabular}{@{}c@{}}0.4340 \\ \raisebox{1.5ex}{\tiny$\big(0.4325\big)$}\end{tabular} & \begin{tabular}{@{}c@{}}0.2023 \\ \raisebox{1.5ex}{\tiny$\big(0.2035\big)$}\end{tabular} & \begin{tabular}{@{}c@{}}0.4829 \\ \raisebox{1.5ex}{\tiny$\big(0.4792\big)$}\end{tabular} \\
    & Train MSD ($\kappa=1$) & \begin{tabular}{@{}c@{}}0.2819 \\ \raisebox{1.5ex}{\tiny$\big(0.2846\big)$}\end{tabular} & \begin{tabular}{@{}c@{}}0.2173 \\ \raisebox{1.5ex}{\tiny$\big(0.2186\big)$}\end{tabular} & \begin{tabular}{@{}c@{}}0.6028 \\ \raisebox{1.5ex}{\tiny$\big(0.6043\big)$}\end{tabular} & \begin{tabular}{@{}c@{}}0.7133 \\ \raisebox{1.5ex}{\tiny$\big(0.7134\big)$}\end{tabular} & \begin{tabular}{@{}c@{}}0.3660 \\ \raisebox{1.5ex}{\tiny$\big(0.3687\big)$}\end{tabular} & \begin{tabular}{@{}c@{}}0.7740 \\ \raisebox{1.5ex}{\tiny$\big(0.7715\big)$}\end{tabular} \\
    & Test Exp Val & \begin{tabular}{@{}c@{}}0.2075 \\ \raisebox{1.5ex}{\tiny$\big(0.2128\big)$}\end{tabular} & \begin{tabular}{@{}c@{}}0.1653 \\ \raisebox{1.5ex}{\tiny$\big(0.1817\big)$}\end{tabular} & \begin{tabular}{@{}c@{}}0.4583 \\ \raisebox{1.5ex}{\tiny$\big(0.4691\big)$}\end{tabular} & \begin{tabular}{@{}c@{}}0.5332 \\ \raisebox{1.5ex}{\tiny$\big(0.5359\big)$}\end{tabular} & \begin{tabular}{@{}c@{}}0.2860 \\ \raisebox{1.5ex}{\tiny$\big(0.2945\big)$}\end{tabular} & \begin{tabular}{@{}c@{}}0.5815 \\ \raisebox{1.5ex}{\tiny$\big(0.5837\big)$}\end{tabular} \\
    & Test MSD ($\kappa=1$) & \begin{tabular}{@{}c@{}}0.3786 \\ \raisebox{1.5ex}{\tiny$\big(0.3884\big)$}\end{tabular} & \begin{tabular}{@{}c@{}}0.3062 \\ \raisebox{1.5ex}{\tiny$\big(0.3376\big)$}\end{tabular} & \begin{tabular}{@{}c@{}}0.7718 \\ \raisebox{1.5ex}{\tiny$\big(0.7920\big)$}\end{tabular} & \begin{tabular}{@{}c@{}}0.8601 \\ \raisebox{1.5ex}{\tiny$\big(0.8668\big)$}\end{tabular} & \begin{tabular}{@{}c@{}}0.5098 \\ \raisebox{1.5ex}{\tiny$\big(0.5257\big)$}\end{tabular} & \begin{tabular}{@{}c@{}}0.9128 \\ \raisebox{1.5ex}{\tiny$\big(0.9194\big)$}\end{tabular} \\
    \midrule
    \multirow{8}{*}{$\kappa = 0.35$} 
    & Train Exp Val & \begin{tabular}{@{}c@{}}0.1455 \\ \raisebox{1.5ex}{\tiny$\big(0.1536\big)$}\end{tabular} & \begin{tabular}{@{}c@{}}0.1123 \\ \raisebox{1.5ex}{\tiny$\big(0.1170\big)$}\end{tabular} & \begin{tabular}{@{}c@{}}0.3495 \\ \raisebox{1.5ex}{\tiny$\big(0.3496\big)$}\end{tabular} & \begin{tabular}{@{}c@{}}0.4394 \\ \raisebox{1.5ex}{\tiny$\big(0.4325\big)$}\end{tabular} & \begin{tabular}{@{}c@{}}0.1966 \\ \raisebox{1.5ex}{\tiny$\big(0.2035\big)$}\end{tabular} & \begin{tabular}{@{}c@{}}0.5027 \\ \raisebox{1.5ex}{\tiny$\big(0.4792\big)$}\end{tabular} \\
    & Train MSD ($\kappa=1$) & \begin{tabular}{@{}c@{}}0.2681 \\ \raisebox{1.5ex}{\tiny$\big(0.2846\big)$}\end{tabular} & \begin{tabular}{@{}c@{}}0.2086 \\ \raisebox{1.5ex}{\tiny$\big(0.2186\big)$}\end{tabular} & \begin{tabular}{@{}c@{}}0.5950 \\ \raisebox{1.5ex}{\tiny$\big(0.6043\big)$}\end{tabular} & \begin{tabular}{@{}c@{}}0.7094 \\ \raisebox{1.5ex}{\tiny$\big(0.7134\big)$}\end{tabular} & \begin{tabular}{@{}c@{}}0.3530 \\ \raisebox{1.5ex}{\tiny$\big(0.3687\big)$}\end{tabular} & \begin{tabular}{@{}c@{}}0.7867 \\ \raisebox{1.5ex}{\tiny$\big(0.7715\big)$}\end{tabular} \\
    & Test Exp Val & \begin{tabular}{@{}c@{}}0.1982 \\ \raisebox{1.5ex}{\tiny$\big(0.2128\big)$}\end{tabular} & \begin{tabular}{@{}c@{}}0.1608 \\ \raisebox{1.5ex}{\tiny$\big(0.1817\big)$}\end{tabular} & \begin{tabular}{@{}c@{}}0.4527 \\ \raisebox{1.5ex}{\tiny$\big(0.4691\big)$}\end{tabular} & \begin{tabular}{@{}c@{}}0.5323 \\ \raisebox{1.5ex}{\tiny$\big(0.5359\big)$}\end{tabular} & \begin{tabular}{@{}c@{}}0.2764 \\ \raisebox{1.5ex}{\tiny$\big(0.2945\big)$}\end{tabular} & \begin{tabular}{@{}c@{}}0.5942 \\ \raisebox{1.5ex}{\tiny$\big(0.5837\big)$}\end{tabular} \\
    & Test MSD ($\kappa=1$) & \begin{tabular}{@{}c@{}}0.3612 \\ \raisebox{1.5ex}{\tiny$\big(0.3884\big)$}\end{tabular} & \begin{tabular}{@{}c@{}}0.2973 \\ \raisebox{1.5ex}{\tiny$\big(0.3376\big)$}\end{tabular} & \begin{tabular}{@{}c@{}}0.7562 \\ \raisebox{1.5ex}{\tiny$\big(0.7920\big)$}\end{tabular} & \begin{tabular}{@{}c@{}}0.8479 \\ \raisebox{1.5ex}{\tiny$\big(0.8668\big)$}\end{tabular} & \begin{tabular}{@{}c@{}}0.4912 \\ \raisebox{1.5ex}{\tiny$\big(0.5257\big)$}\end{tabular} & \begin{tabular}{@{}c@{}}0.9163 \\ \raisebox{1.5ex}{\tiny$\big(0.9194\big)$}\end{tabular} \\
    \midrule
    \multirow{8}{*}{$\kappa = 0.65$} 
    & Train Exp Val & \begin{tabular}{@{}c@{}}0.1396 \\ \raisebox{1.5ex}{\tiny$\big(0.1536\big)$}\end{tabular} & \begin{tabular}{@{}c@{}}0.1091 \\ \raisebox{1.5ex}{\tiny$\big(0.1170\big)$}\end{tabular} & \begin{tabular}{@{}c@{}}0.3510 \\ \raisebox{1.5ex}{\tiny$\big(0.3496\big)$}\end{tabular} & \begin{tabular}{@{}c@{}}0.4457 \\ \raisebox{1.5ex}{\tiny$\big(0.4325\big)$}\end{tabular} & \begin{tabular}{@{}c@{}}0.1920 \\ \raisebox{1.5ex}{\tiny$\big(0.2035\big)$}\end{tabular} & \begin{tabular}{@{}c@{}}0.5255 \\ \raisebox{1.5ex}{\tiny$\big(0.4792\big)$}\end{tabular} \\
    & Train MSD ($\kappa=1$) & \begin{tabular}{@{}c@{}}0.2563 \\ \raisebox{1.5ex}{\tiny$\big(0.2846\big)$}\end{tabular} & \begin{tabular}{@{}c@{}}0.2020 \\ \raisebox{1.5ex}{\tiny$\big(0.2186\big)$}\end{tabular} & \begin{tabular}{@{}c@{}}0.5908 \\ \raisebox{1.5ex}{\tiny$\big(0.6043\big)$}\end{tabular} & \begin{tabular}{@{}c@{}}0.7082 \\ \raisebox{1.5ex}{\tiny$\big(0.7134\big)$}\end{tabular} & \begin{tabular}{@{}c@{}}0.3428 \\ \raisebox{1.5ex}{\tiny$\big(0.3687\big)$}\end{tabular} & \begin{tabular}{@{}c@{}}0.8034 \\ \raisebox{1.5ex}{\tiny$\big(0.7715\big)$}\end{tabular} \\
    & Test Exp Val & \begin{tabular}{@{}c@{}}0.1909 \\ \raisebox{1.5ex}{\tiny$\big(0.2128\big)$}\end{tabular} & \begin{tabular}{@{}c@{}}0.1574 \\ \raisebox{1.5ex}{\tiny$\big(0.1817\big)$}\end{tabular} & \begin{tabular}{@{}c@{}}0.4505 \\ \raisebox{1.5ex}{\tiny$\big(0.4691\big)$}\end{tabular} & \begin{tabular}{@{}c@{}}0.5338 \\ \raisebox{1.5ex}{\tiny$\big(0.5359\big)$}\end{tabular} & \begin{tabular}{@{}c@{}}0.2682 \\ \raisebox{1.5ex}{\tiny$\big(0.2945\big)$}\end{tabular} & \begin{tabular}{@{}c@{}}0.6121 \\ \raisebox{1.5ex}{\tiny$\big(0.5837\big)$}\end{tabular} \\
    & Test MSD ($\kappa=1$) & \begin{tabular}{@{}c@{}}0.3474 \\ \raisebox{1.5ex}{\tiny$\big(0.3884\big)$}\end{tabular} & \begin{tabular}{@{}c@{}}0.2904 \\ \raisebox{1.5ex}{\tiny$\big(0.3376\big)$}\end{tabular} & \begin{tabular}{@{}c@{}}0.7477 \\ \raisebox{1.5ex}{\tiny$\big(0.7920\big)$}\end{tabular} & \begin{tabular}{@{}c@{}}0.8401 \\ \raisebox{1.5ex}{\tiny$\big(0.8668\big)$}\end{tabular} & \begin{tabular}{@{}c@{}}0.4755 \\ \raisebox{1.5ex}{\tiny$\big(0.5257\big)$}\end{tabular} & \begin{tabular}{@{}c@{}}0.9266 \\ \raisebox{1.5ex}{\tiny$\big(0.9194\big)$}\end{tabular} \\
    \midrule
    \multirow{8}{*}{$\kappa = 0.95$} 
    & Train Exp Val & \begin{tabular}{@{}c@{}}0.1349 \\ \raisebox{1.5ex}{\tiny$\big(0.1536\big)$}\end{tabular} & \begin{tabular}{@{}c@{}}0.1065 \\ \raisebox{1.5ex}{\tiny$\big(0.1170\big)$}\end{tabular} & \begin{tabular}{@{}c@{}}0.3532 \\ \raisebox{1.5ex}{\tiny$\big(0.3496\big)$}\end{tabular} & \begin{tabular}{@{}c@{}}0.4539 \\ \raisebox{1.5ex}{\tiny$\big(0.4325\big)$}\end{tabular} & \begin{tabular}{@{}c@{}}0.1882 \\ \raisebox{1.5ex}{\tiny$\big(0.2035\big)$}\end{tabular} & \begin{tabular}{@{}c@{}}0.5505 \\ \raisebox{1.5ex}{\tiny$\big(0.4792\big)$}\end{tabular} \\
    & Train MSD ($\kappa=1$) & \begin{tabular}{@{}c@{}}0.2469 \\ \raisebox{1.5ex}{\tiny$\big(0.2846\big)$}\end{tabular} & \begin{tabular}{@{}c@{}}0.1965 \\ \raisebox{1.5ex}{\tiny$\big(0.2186\big)$}\end{tabular} & \begin{tabular}{@{}c@{}}0.5884 \\ \raisebox{1.5ex}{\tiny$\big(0.6043\big)$}\end{tabular} & \begin{tabular}{@{}c@{}}0.7093 \\ \raisebox{1.5ex}{\tiny$\big(0.7134\big)$}\end{tabular} & \begin{tabular}{@{}c@{}}0.3343 \\ \raisebox{1.5ex}{\tiny$\big(0.3687\big)$}\end{tabular} & \begin{tabular}{@{}c@{}}0.8219 \\ \raisebox{1.5ex}{\tiny$\big(0.7715\big)$}\end{tabular} \\
    & Test Exp Val & \begin{tabular}{@{}c@{}}0.1844 \\ \raisebox{1.5ex}{\tiny$\big(0.2128\big)$}\end{tabular} & \begin{tabular}{@{}c@{}}0.1545 \\ \raisebox{1.5ex}{\tiny$\big(0.1817\big)$}\end{tabular} & \begin{tabular}{@{}c@{}}0.4493 \\ \raisebox{1.5ex}{\tiny$\big(0.4691\big)$}\end{tabular} & \begin{tabular}{@{}c@{}}0.5378 \\ \raisebox{1.5ex}{\tiny$\big(0.5359\big)$}\end{tabular} & \begin{tabular}{@{}c@{}}0.2620 \\ \raisebox{1.5ex}{\tiny$\big(0.2945\big)$}\end{tabular} & \begin{tabular}{@{}c@{}}0.6323 \\ \raisebox{1.5ex}{\tiny$\big(0.5837\big)$}\end{tabular} \\
    & Test MSD ($\kappa=1$) & \begin{tabular}{@{}c@{}}0.3354 \\ \raisebox{1.5ex}{\tiny$\big(0.3884\big)$}\end{tabular} & \begin{tabular}{@{}c@{}}0.2848 \\ \raisebox{1.5ex}{\tiny$\big(0.3376\big)$}\end{tabular} & \begin{tabular}{@{}c@{}}0.7408 \\ \raisebox{1.5ex}{\tiny$\big(0.7920\big)$}\end{tabular} & \begin{tabular}{@{}c@{}}0.8356 \\ \raisebox{1.5ex}{\tiny$\big(0.8668\big)$}\end{tabular} & \begin{tabular}{@{}c@{}}0.4635 \\ \raisebox{1.5ex}{\tiny$\big(0.5257\big)$}\end{tabular} & \begin{tabular}{@{}c@{}}0.9391 \\ \raisebox{1.5ex}{\tiny$\big(0.9194\big)$}\end{tabular} \\
    \bottomrule
  \end{tabular}
  \caption{Risk table of the experiments with 90\% features removed. The numbers in parentheses are associated with the risk-neutral model; the numbers outside come from the risk-averse method.}
  \label{tab:risk-remove}
\end{table}

Table.\ref{tab:risk-remove}, again shows that the risk-averse method performs significantly better, especially on the out-of-sample test data. It is also interesting to see that, the risk of some of the classes is relatively larger than the rest, namely classes 3, 4, and 6. These classes correspond to numbers 2,3 and 5. This means that these digits are more difficult, in other words, more risky to classify. 

\begin{figure}[h!]
\centering
\includegraphics[width=0.95\linewidth]{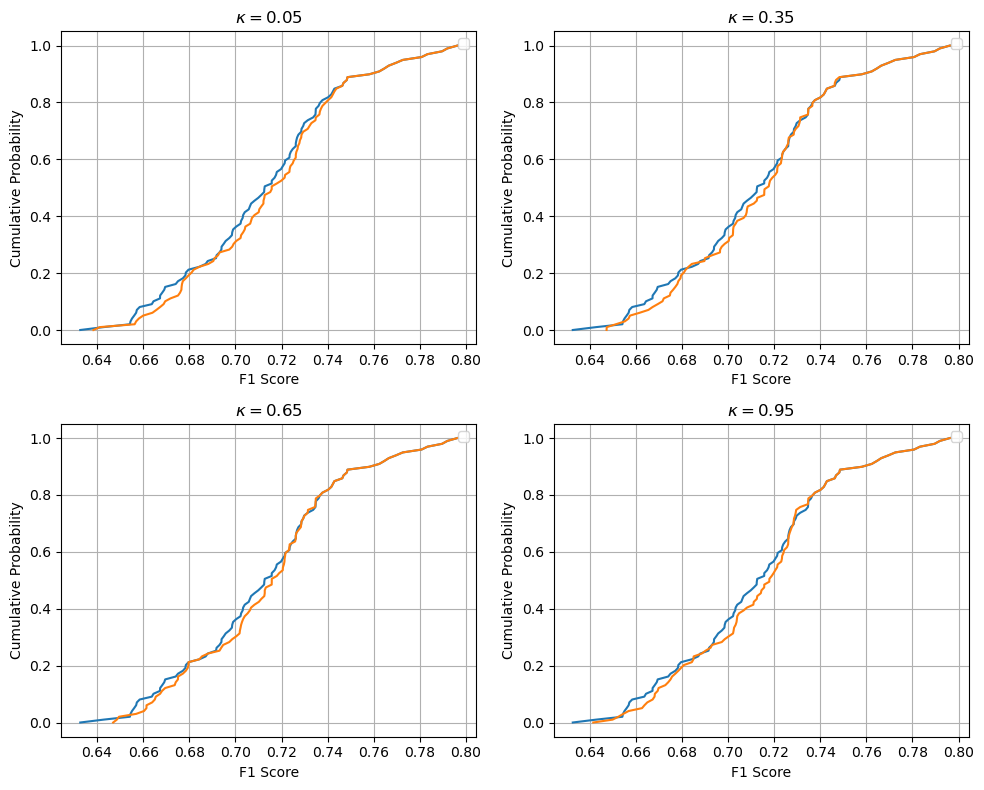}
\caption{CDF of the average F1 scores comparing the risk-neutral baseline (blue) and the risk-averse method (orange) with different risk levels $\kappa$ using 6 classes of small-sized samples of feature-removed data.}
\label{fig:f1-6small-remove}
\end{figure}

\subsection{Size-limited data with removed features}

In the third set of experiments, we randomly select only 150 samples for each experiment, and then randomly remove $90\%$ of the features. Moreover, we increase the test ratio to 0.5, which means only 75 samples are training samples. Hence, the methods have very limited information from each image. Even though we are surprised to see that both methods are still able to classify the images under such an extreme situation, the results are definitely far from good, which indicates the existence of risk that lack of information creates.
In Fig.\ref{fig:f1-6small-remove}, it is obvious that the difference between the two models is much more significant than the difference shown in Fig.\ref{fig:f1-remove}. This indicates that, compared to the risk-neutral model, the risk-averse model has better performance when very small amount of data is available.

\begin{figure}[h!]
\centering
\includegraphics[width=0.95\linewidth]{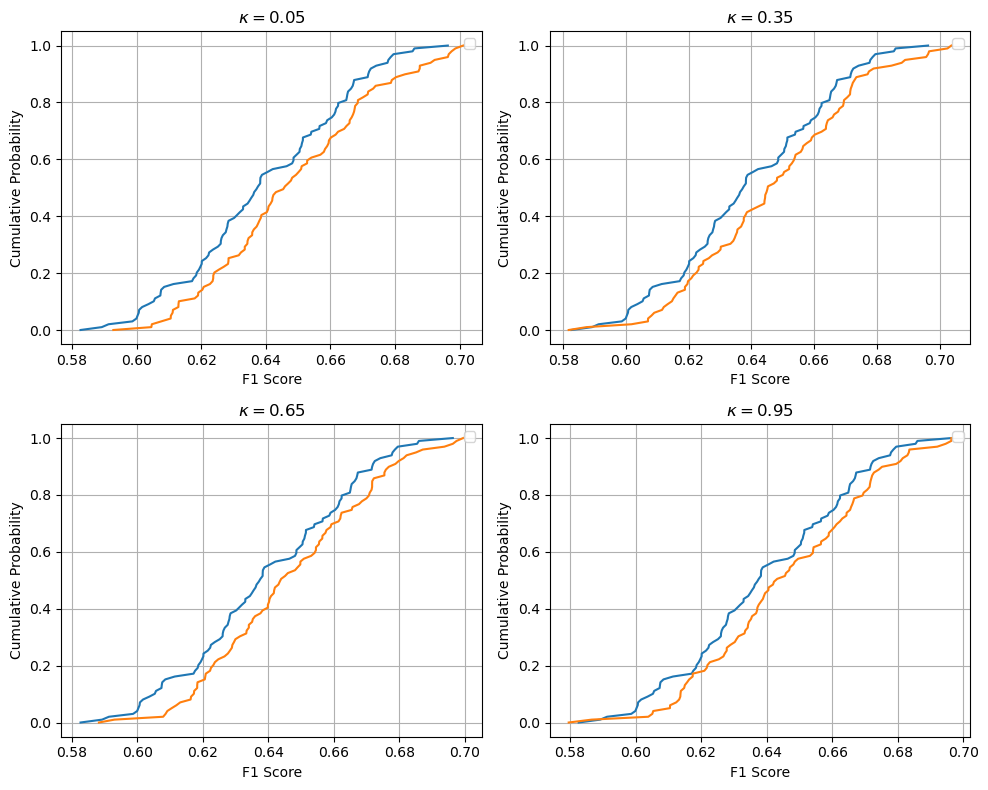}
\caption{CDF of the average F1 scores comparing the risk-neutral baseline (blue) and the risk-averse method (orange) with different risk levels $\kappa$ using 10 classes of small-sized samples of feature-removed data.}
\label{fig:f1-10small-remove}
\end{figure}

\begin{table}[h!]
  \centering
  \small
  \begin{tabular}{cl*{10}{p{0.4cm}}}
    \toprule
    & Error & \multicolumn{10}{c}{Class} \\
    \cmidrule(r){3-12}
    & Statistics & 1 & 2 & 3 & 4 & 5 & 6 & 7 & 8 & 9 & 10 \\
    \midrule
    \multirow{8}{*}{$\kappa = 0.05$} 
    & Train Exp Val & \begin{tabular}{@{}c@{}}0.05 \\ \raisebox{1.5ex}{\tiny$(0.04)$}\end{tabular} & \begin{tabular}{@{}c@{}}0.03 \\ \raisebox{1.5ex}{\tiny$(0.02)$}\end{tabular} & \begin{tabular}{@{}c@{}}0.15 \\ \raisebox{1.5ex}{\tiny$(0.14)$}\end{tabular} & \begin{tabular}{@{}c@{}}0.22 \\ \raisebox{1.5ex}{\tiny$(0.21)$}\end{tabular} & \begin{tabular}{@{}c@{}}0.17 \\ \raisebox{1.5ex}{\tiny$(0.16)$}\end{tabular} & \begin{tabular}{@{}c@{}}0.29 \\ \raisebox{1.5ex}{\tiny$(0.26)$}\end{tabular} & \begin{tabular}{@{}c@{}}0.05 \\ \raisebox{1.5ex}{\tiny$(0.04)$}\end{tabular} & \begin{tabular}{@{}c@{}}0.14 \\ \raisebox{1.5ex}{\tiny$(0.13)$}\end{tabular} & \begin{tabular}{@{}c@{}}0.23 \\ \raisebox{1.5ex}{\tiny$(0.21)$}\end{tabular} & \begin{tabular}{@{}c@{}}0.25 \\ \raisebox{1.5ex}{\tiny$(0.23)$}\end{tabular} \\
    & Train MSD ($\kappa=1$) & \begin{tabular}{@{}c@{}}0.09 \\ \raisebox{1.5ex}{\tiny$(0.08)$}\end{tabular} & \begin{tabular}{@{}c@{}}0.05 \\ \raisebox{1.5ex}{\tiny$(0.04)$}\end{tabular} & \begin{tabular}{@{}c@{}}0.28 \\ \raisebox{1.5ex}{\tiny$(0.26)$}\end{tabular} & \begin{tabular}{@{}c@{}}0.39 \\ \raisebox{1.5ex}{\tiny$(0.38)$}\end{tabular} & \begin{tabular}{@{}c@{}}0.31 \\ \raisebox{1.5ex}{\tiny$(0.30)$}\end{tabular} & \begin{tabular}{@{}c@{}}0.50 \\ \raisebox{1.5ex}{\tiny$(0.46)$}\end{tabular} & \begin{tabular}{@{}c@{}}0.09 \\ \raisebox{1.5ex}{\tiny$(0.08)$}\end{tabular} & \begin{tabular}{@{}c@{}}0.26 \\ \raisebox{1.5ex}{\tiny$(0.24)$}\end{tabular} & \begin{tabular}{@{}c@{}}0.42 \\ \raisebox{1.5ex}{\tiny$(0.38)$}\end{tabular} & \begin{tabular}{@{}c@{}}0.45 \\ \raisebox{1.5ex}{\tiny$(0.41)$}\end{tabular} \\
    & Test Exp Val & \begin{tabular}{@{}c@{}}2.33 \\ \raisebox{1.5ex}{\tiny$(3.99)$}\end{tabular} & \begin{tabular}{@{}c@{}}1.93 \\ \raisebox{1.5ex}{\tiny$(3.20)$}\end{tabular} & \begin{tabular}{@{}c@{}}4.93 \\ \raisebox{1.5ex}{\tiny$(8.23)$}\end{tabular} & \begin{tabular}{@{}c@{}}3.95 \\ \raisebox{1.5ex}{\tiny$(6.31)$}\end{tabular} & \begin{tabular}{@{}c@{}}3.98 \\ \raisebox{1.5ex}{\tiny$(6.82)$}\end{tabular} & \begin{tabular}{@{}c@{}}4.29 \\ \raisebox{1.5ex}{\tiny$(6.80)$}\end{tabular} & \begin{tabular}{@{}c@{}}3.14 \\ \raisebox{1.5ex}{\tiny$(5.44)$}\end{tabular} & \begin{tabular}{@{}c@{}}3.12 \\ \raisebox{1.5ex}{\tiny$(5.10)$}\end{tabular} & \begin{tabular}{@{}c@{}}3.37 \\ \raisebox{1.5ex}{\tiny$(5.29)$}\end{tabular} & \begin{tabular}{@{}c@{}}3.29 \\ \raisebox{1.5ex}{\tiny$(5.28)$}\end{tabular} \\
    & Test MSD ($\kappa=1$) & \begin{tabular}{@{}c@{}}4.09 \\ \raisebox{1.5ex}{\tiny$(7.04)$}\end{tabular} & \begin{tabular}{@{}c@{}}3.53 \\ \raisebox{1.5ex}{\tiny$(5.86)$}\end{tabular} & \begin{tabular}{@{}c@{}}8.10 \\ \raisebox{1.5ex}{\tiny$(13.63)$}\end{tabular} & \begin{tabular}{@{}c@{}}6.40 \\ \raisebox{1.5ex}{\tiny$(10.35)$}\end{tabular} & \begin{tabular}{@{}c@{}}6.48 \\ \raisebox{1.5ex}{\tiny$(11.28)$}\end{tabular} & \begin{tabular}{@{}c@{}}6.70 \\ \raisebox{1.5ex}{\tiny$(10.79)$}\end{tabular} & \begin{tabular}{@{}c@{}}5.47 \\ \raisebox{1.5ex}{\tiny$(9.52)$}\end{tabular} & \begin{tabular}{@{}c@{}}5.28 \\ \raisebox{1.5ex}{\tiny$(8.72)$}\end{tabular} & \begin{tabular}{@{}c@{}}5.46 \\ \raisebox{1.5ex}{\tiny$(8.69)$}\end{tabular} & \begin{tabular}{@{}c@{}}5.36 \\ \raisebox{1.5ex}{\tiny$(8.72)$}\end{tabular} \\
    \midrule
    \multirow{8}{*}{$\kappa = 0.35$}
    & Train Exp Val & \begin{tabular}{@{}c@{}}0.04 \\ \raisebox{1.5ex}{\tiny$(0.04)$}\end{tabular} & \begin{tabular}{@{}c@{}}0.02 \\ \raisebox{1.5ex}{\tiny$(0.02)$}\end{tabular} & \begin{tabular}{@{}c@{}}0.15 \\ \raisebox{1.5ex}{\tiny$(0.14)$}\end{tabular} & \begin{tabular}{@{}c@{}}0.21 \\ \raisebox{1.5ex}{\tiny$(0.21)$}\end{tabular} & \begin{tabular}{@{}c@{}}0.17 \\ \raisebox{1.5ex}{\tiny$(0.16)$}\end{tabular} & \begin{tabular}{@{}c@{}}0.30 \\ \raisebox{1.5ex}{\tiny$(0.26)$}\end{tabular} & \begin{tabular}{@{}c@{}}0.04 \\ \raisebox{1.5ex}{\tiny$(0.04)$}\end{tabular} & \begin{tabular}{@{}c@{}}0.14 \\ \raisebox{1.5ex}{\tiny$(0.13)$}\end{tabular} & \begin{tabular}{@{}c@{}}0.23 \\ \raisebox{1.5ex}{\tiny$(0.21)$}\end{tabular} & \begin{tabular}{@{}c@{}}0.26 \\ \raisebox{1.5ex}{\tiny$(0.23)$}\end{tabular} \\
    & Train MSD ($\kappa=1$) & \begin{tabular}{@{}c@{}}0.08 \\ \raisebox{1.5ex}{\tiny$(0.08)$}\end{tabular} & \begin{tabular}{@{}c@{}}0.05 \\ \raisebox{1.5ex}{\tiny$(0.04)$}\end{tabular} & \begin{tabular}{@{}c@{}}0.27 \\ \raisebox{1.5ex}{\tiny$(0.26)$}\end{tabular} & \begin{tabular}{@{}c@{}}0.38 \\ \raisebox{1.5ex}{\tiny$(0.38)$}\end{tabular} & \begin{tabular}{@{}c@{}}0.29 \\ \raisebox{1.5ex}{\tiny$(0.30)$}\end{tabular} & \begin{tabular}{@{}c@{}}0.51 \\ \raisebox{1.5ex}{\tiny$(0.46)$}\end{tabular} & \begin{tabular}{@{}c@{}}0.08 \\ \raisebox{1.5ex}{\tiny$(0.08)$}\end{tabular} & \begin{tabular}{@{}c@{}}0.25 \\ \raisebox{1.5ex}{\tiny$(0.24)$}\end{tabular} & \begin{tabular}{@{}c@{}}0.41 \\ \raisebox{1.5ex}{\tiny$(0.38)$}\end{tabular} & \begin{tabular}{@{}c@{}}0.45 \\ \raisebox{1.5ex}{\tiny$(0.41)$}\end{tabular} \\
    & Test Exp Val & \begin{tabular}{@{}c@{}}2.42 \\ \raisebox{1.5ex}{\tiny$(3.99)$}\end{tabular} & \begin{tabular}{@{}c@{}}2.00 \\ \raisebox{1.5ex}{\tiny$(3.20)$}\end{tabular} & \begin{tabular}{@{}c@{}}5.11 \\ \raisebox{1.5ex}{\tiny$(8.23)$}\end{tabular} & \begin{tabular}{@{}c@{}}4.06 \\ \raisebox{1.5ex}{\tiny$(6.31)$}\end{tabular} & \begin{tabular}{@{}c@{}}4.15 \\ \raisebox{1.5ex}{\tiny$(6.82)$}\end{tabular} & \begin{tabular}{@{}c@{}}4.45 \\ \raisebox{1.5ex}{\tiny$(6.80)$}\end{tabular} & \begin{tabular}{@{}c@{}}3.27 \\ \raisebox{1.5ex}{\tiny$(5.44)$}\end{tabular} & \begin{tabular}{@{}c@{}}3.24 \\ \raisebox{1.5ex}{\tiny$(5.10)$}\end{tabular} & \begin{tabular}{@{}c@{}}3.48 \\ \raisebox{1.5ex}{\tiny$(5.29)$}\end{tabular} & \begin{tabular}{@{}c@{}}3.41 \\ \raisebox{1.5ex}{\tiny$(5.28)$}\end{tabular} \\
    & Test MSD ($\kappa=1$) & \begin{tabular}{@{}c@{}}4.25 \\ \raisebox{1.5ex}{\tiny$(7.04)$}\end{tabular} & \begin{tabular}{@{}c@{}}3.65 \\ \raisebox{1.5ex}{\tiny$(5.86)$}\end{tabular} & \begin{tabular}{@{}c@{}}8.40 \\ \raisebox{1.5ex}{\tiny$(13.63)$}\end{tabular} & \begin{tabular}{@{}c@{}}6.58 \\ \raisebox{1.5ex}{\tiny$(10.35)$}\end{tabular} & \begin{tabular}{@{}c@{}}6.77 \\ \raisebox{1.5ex}{\tiny$(11.28)$}\end{tabular} & \begin{tabular}{@{}c@{}}6.95 \\ \raisebox{1.5ex}{\tiny$(10.79)$}\end{tabular} & \begin{tabular}{@{}c@{}}5.70 \\ \raisebox{1.5ex}{\tiny$(9.52)$}\end{tabular} & \begin{tabular}{@{}c@{}}5.48 \\ \raisebox{1.5ex}{\tiny$(8.72)$}\end{tabular} & \begin{tabular}{@{}c@{}}5.65 \\ \raisebox{1.5ex}{\tiny$(8.69)$}\end{tabular} & \begin{tabular}{@{}c@{}}5.55 \\ \raisebox{1.5ex}{\tiny$(8.72)$}\end{tabular} \\
    \midrule
    \multirow{8}{*}{$\kappa = 0.65$}
    & Train Exp Val & \begin{tabular}{@{}c@{}}0.04 \\ \raisebox{1.5ex}{\tiny$(0.04)$}\end{tabular} & \begin{tabular}{@{}c@{}}0.02 \\ \raisebox{1.5ex}{\tiny$(0.02)$}\end{tabular} & \begin{tabular}{@{}c@{}}0.14 \\ \raisebox{1.5ex}{\tiny$(0.14)$}\end{tabular} & \begin{tabular}{@{}c@{}}0.22 \\ \raisebox{1.5ex}{\tiny$(0.21)$}\end{tabular} & \begin{tabular}{@{}c@{}}0.16 \\ \raisebox{1.5ex}{\tiny$(0.16)$}\end{tabular} & \begin{tabular}{@{}c@{}}0.32 \\ \raisebox{1.5ex}{\tiny$(0.26)$}\end{tabular} & \begin{tabular}{@{}c@{}}0.04 \\ \raisebox{1.5ex}{\tiny$(0.04)$}\end{tabular} & \begin{tabular}{@{}c@{}}0.13 \\ \raisebox{1.5ex}{\tiny$(0.13)$}\end{tabular} & \begin{tabular}{@{}c@{}}0.24 \\ \raisebox{1.5ex}{\tiny$(0.21)$}\end{tabular} & \begin{tabular}{@{}c@{}}0.28 \\ \raisebox{1.5ex}{\tiny$(0.23)$}\end{tabular} \\
    & Train MSD ($\kappa=1$) & \begin{tabular}{@{}c@{}}0.08 \\ \raisebox{1.5ex}{\tiny$(0.08)$}\end{tabular} & \begin{tabular}{@{}c@{}}0.04 \\ \raisebox{1.5ex}{\tiny$(0.04)$}\end{tabular} & \begin{tabular}{@{}c@{}}0.26 \\ \raisebox{1.5ex}{\tiny$(0.26)$}\end{tabular} & \begin{tabular}{@{}c@{}}0.37 \\ \raisebox{1.5ex}{\tiny$(0.38)$}\end{tabular} & \begin{tabular}{@{}c@{}}0.29 \\ \raisebox{1.5ex}{\tiny$(0.30)$}\end{tabular} & \begin{tabular}{@{}c@{}}0.51 \\ \raisebox{1.5ex}{\tiny$(0.46)$}\end{tabular} & \begin{tabular}{@{}c@{}}0.08 \\ \raisebox{1.5ex}{\tiny$(0.08)$}\end{tabular} & \begin{tabular}{@{}c@{}}0.24 \\ \raisebox{1.5ex}{\tiny$(0.24)$}\end{tabular} & \begin{tabular}{@{}c@{}}0.40 \\ \raisebox{1.5ex}{\tiny$(0.38)$}\end{tabular} & \begin{tabular}{@{}c@{}}0.45 \\ \raisebox{1.5ex}{\tiny$(0.41)$}\end{tabular} \\
    & Test Exp Val & \begin{tabular}{@{}c@{}}2.49 \\ \raisebox{1.5ex}{\tiny$(3.99)$}\end{tabular} & \begin{tabular}{@{}c@{}}2.06 \\ \raisebox{1.5ex}{\tiny$(3.20)$}\end{tabular} & \begin{tabular}{@{}c@{}}5.26 \\ \raisebox{1.5ex}{\tiny$(8.23)$}\end{tabular} & \begin{tabular}{@{}c@{}}4.17 \\ \raisebox{1.5ex}{\tiny$(6.31)$}\end{tabular} & \begin{tabular}{@{}c@{}}4.33 \\ \raisebox{1.5ex}{\tiny$(6.82)$}\end{tabular} & \begin{tabular}{@{}c@{}}4.57 \\ \raisebox{1.5ex}{\tiny$(6.80)$}\end{tabular} & \begin{tabular}{@{}c@{}}3.40 \\ \raisebox{1.5ex}{\tiny$(5.44)$}\end{tabular} & \begin{tabular}{@{}c@{}}3.35 \\ \raisebox{1.5ex}{\tiny$(5.10)$}\end{tabular} & \begin{tabular}{@{}c@{}}3.57 \\ \raisebox{1.5ex}{\tiny$(5.29)$}\end{tabular} & \begin{tabular}{@{}c@{}}3.52 \\ \raisebox{1.5ex}{\tiny$(5.28)$}\end{tabular} \\
    & Test MSD ($\kappa=1$) & \begin{tabular}{@{}c@{}}4.39 \\ \raisebox{1.5ex}{\tiny$(7.04)$}\end{tabular} & \begin{tabular}{@{}c@{}}3.77 \\ \raisebox{1.5ex}{\tiny$(5.86)$}\end{tabular} & \begin{tabular}{@{}c@{}}8.66 \\ \raisebox{1.5ex}{\tiny$(13.63)$}\end{tabular} & \begin{tabular}{@{}c@{}}6.76 \\ \raisebox{1.5ex}{\tiny$(10.35)$}\end{tabular} & \begin{tabular}{@{}c@{}}7.07 \\ \raisebox{1.5ex}{\tiny$(11.28)$}\end{tabular} & \begin{tabular}{@{}c@{}}7.14 \\ \raisebox{1.5ex}{\tiny$(10.79)$}\end{tabular} & \begin{tabular}{@{}c@{}}5.93 \\ \raisebox{1.5ex}{\tiny$(9.52)$}\end{tabular} & \begin{tabular}{@{}c@{}}5.68 \\ \raisebox{1.5ex}{\tiny$(8.72)$}\end{tabular} & \begin{tabular}{@{}c@{}}5.79 \\ \raisebox{1.5ex}{\tiny$(8.69)$}\end{tabular} & \begin{tabular}{@{}c@{}}5.73 \\ \raisebox{1.5ex}{\tiny$(8.72)$}\end{tabular} \\
    \midrule
    \multirow{8}{*}{$\kappa = 0.95$}
    & Train Exp Val & \begin{tabular}{@{}c@{}}0.04 \\ \raisebox{1.5ex}{\tiny$(0.04)$}\end{tabular} & \begin{tabular}{@{}c@{}}0.02 \\ \raisebox{1.5ex}{\tiny$(0.02)$}\end{tabular} & \begin{tabular}{@{}c@{}}0.14 \\ \raisebox{1.5ex}{\tiny$(0.14)$}\end{tabular} & \begin{tabular}{@{}c@{}}0.22 \\ \raisebox{1.5ex}{\tiny$(0.21)$}\end{tabular} & \begin{tabular}{@{}c@{}}0.16 \\ \raisebox{1.5ex}{\tiny$(0.16)$}\end{tabular} & \begin{tabular}{@{}c@{}}0.33 \\ \raisebox{1.5ex}{\tiny$(0.26)$}\end{tabular} & \begin{tabular}{@{}c@{}}0.04 \\ \raisebox{1.5ex}{\tiny$(0.04)$}\end{tabular} & \begin{tabular}{@{}c@{}}0.13 \\ \raisebox{1.5ex}{\tiny$(0.13)$}\end{tabular} & \begin{tabular}{@{}c@{}}0.24 \\ \raisebox{1.5ex}{\tiny$(0.21)$}\end{tabular} & \begin{tabular}{@{}c@{}}0.29 \\ \raisebox{1.5ex}{\tiny$(0.23)$}\end{tabular} \\
    & Train MSD ($\kappa=1$) & \begin{tabular}{@{}c@{}}0.07 \\ \raisebox{1.5ex}{\tiny$(0.08)$}\end{tabular} & \begin{tabular}{@{}c@{}}0.04 \\ \raisebox{1.5ex}{\tiny$(0.04)$}\end{tabular} & \begin{tabular}{@{}c@{}}0.25 \\ \raisebox{1.5ex}{\tiny$(0.26)$}\end{tabular} & \begin{tabular}{@{}c@{}}0.37 \\ \raisebox{1.5ex}{\tiny$(0.38)$}\end{tabular} & \begin{tabular}{@{}c@{}}0.28 \\ \raisebox{1.5ex}{\tiny$(0.30)$}\end{tabular} & \begin{tabular}{@{}c@{}}0.52 \\ \raisebox{1.5ex}{\tiny$(0.46)$}\end{tabular} & \begin{tabular}{@{}c@{}}0.07 \\ \raisebox{1.5ex}{\tiny$(0.08)$}\end{tabular} & \begin{tabular}{@{}c@{}}0.23 \\ \raisebox{1.5ex}{\tiny$(0.24)$}\end{tabular} & \begin{tabular}{@{}c@{}}0.40 \\ \raisebox{1.5ex}{\tiny$(0.38)$}\end{tabular} & \begin{tabular}{@{}c@{}}0.46 \\ \raisebox{1.5ex}{\tiny$(0.41)$}\end{tabular} \\
    & Test Exp Val & \begin{tabular}{@{}c@{}}2.56 \\ \raisebox{1.5ex}{\tiny$(3.99)$}\end{tabular} & \begin{tabular}{@{}c@{}}2.12 \\ \raisebox{1.5ex}{\tiny$(3.20)$}\end{tabular} & \begin{tabular}{@{}c@{}}5.40 \\ \raisebox{1.5ex}{\tiny$(8.23)$}\end{tabular} & \begin{tabular}{@{}c@{}}4.25 \\ \raisebox{1.5ex}{\tiny$(6.31)$}\end{tabular} & \begin{tabular}{@{}c@{}}4.44 \\ \raisebox{1.5ex}{\tiny$(6.82)$}\end{tabular} & \begin{tabular}{@{}c@{}}4.67 \\ \raisebox{1.5ex}{\tiny$(6.80)$}\end{tabular} & \begin{tabular}{@{}c@{}}3.50 \\ \raisebox{1.5ex}{\tiny$(5.44)$}\end{tabular} & \begin{tabular}{@{}c@{}}3.43 \\ \raisebox{1.5ex}{\tiny$(5.10)$}\end{tabular} & \begin{tabular}{@{}c@{}}3.65 \\ \raisebox{1.5ex}{\tiny$(5.29)$}\end{tabular} & \begin{tabular}{@{}c@{}}3.59 \\ \raisebox{1.5ex}{\tiny$(5.28)$}\end{tabular} \\
    & Test MSD ($\kappa=1$) & \begin{tabular}{@{}c@{}}4.51 \\ \raisebox{1.5ex}{\tiny$(7.04)$}\end{tabular} & \begin{tabular}{@{}c@{}}3.87 \\ \raisebox{1.5ex}{\tiny$(5.86)$}\end{tabular} & \begin{tabular}{@{}c@{}}8.89 \\ \raisebox{1.5ex}{\tiny$(13.63)$}\end{tabular} & \begin{tabular}{@{}c@{}}6.89 \\ \raisebox{1.5ex}{\tiny$(10.35)$}\end{tabular} & \begin{tabular}{@{}c@{}}7.26 \\ \raisebox{1.5ex}{\tiny$(11.28)$}\end{tabular} & \begin{tabular}{@{}c@{}}7.29 \\ \raisebox{1.5ex}{\tiny$(10.79)$}\end{tabular} & \begin{tabular}{@{}c@{}}6.11 \\ \raisebox{1.5ex}{\tiny$(9.52)$}\end{tabular} & \begin{tabular}{@{}c@{}}5.82 \\ \raisebox{1.5ex}{\tiny$(8.72)$}\end{tabular} & \begin{tabular}{@{}c@{}}5.92 \\ \raisebox{1.5ex}{\tiny$(8.69)$}\end{tabular} & \begin{tabular}{@{}c@{}}5.85 \\ \raisebox{1.5ex}{\tiny$(8.72)$}\end{tabular} \\
    \bottomrule
  \end{tabular}
  \caption{Risk table of the experiments with 10 classes of 150 samples each with 90\% features removed. Numbers in parentheses are from the risk-neutral baseline; numbers outside are from the risk-averse method.}
  \label{tab:risk-10small-remove}
\end{table}

Further, we conduct experiments with all 10 classes of the MNIST dataset while keeping the other settings unchanged. We obtain results in Fig.\ref{fig:f1-10small-remove}. The difference between Fig.\ref{fig:f1-10small-remove} and Fig.\ref{fig:f1-6small-remove} is clear. We can see that with more classes involved, the risk-averse methods shows significantly better F1 score than the baseline model. This observation and the previous analysis  suggest that with more classes, the system becomes more uncertain and more risky, which the risk-averse method is better suited to handle. Another noteworthy observation is that a larger value for the risk level $\kappa$ might be more appropriate in such a scenario. In Fig.\ref{fig:f1-10small-remove}, we can clearly see that the risk levels $\kappa=0.05, 0.35, 0.65$ have very similar performance, whereas the $\kappa=0.95$ one seems to be slightly worse than the rest. 

In Table.\ref{tab:risk-10small-remove}, the risk values with the 10-class experiment are reported. We can clearly see that the numbers here are much larger on average. However, the test risk has been greatly reduced for every class by the risk-averse model. It further shows that coherent risk measures are very useful when dealing with high-risk environments.

\begin{figure}[h]
\centering
  \begin{subfigure}{0.7\textwidth}
    \centering
    \includegraphics[width=\linewidth]{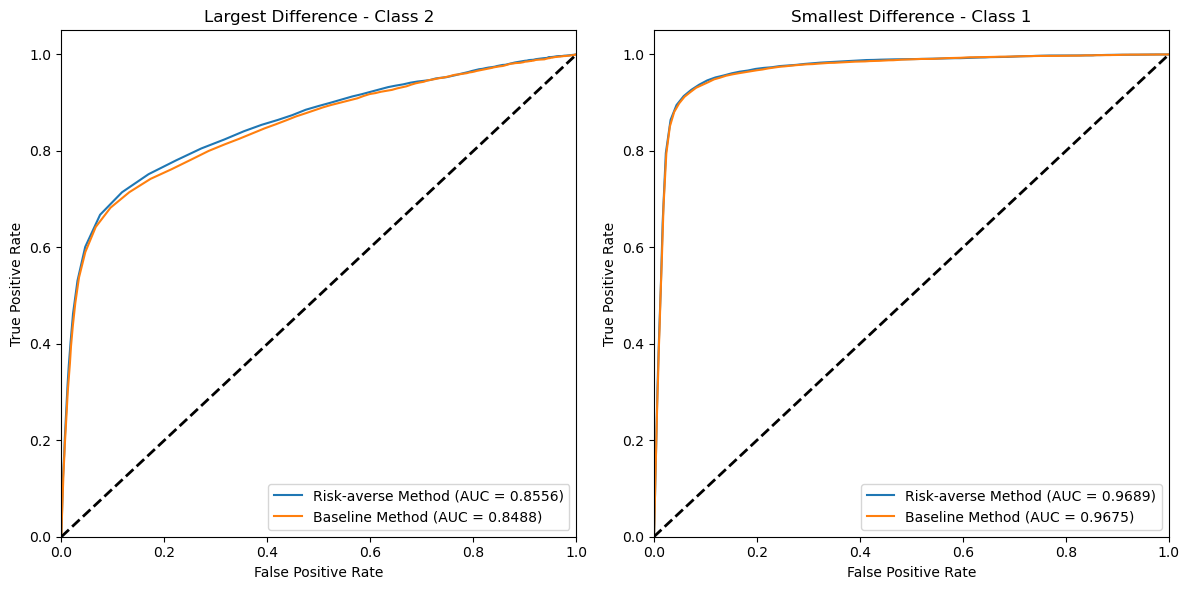}
    \caption{$\kappa=0.95$}
    \label{fig:auc-sub1}
  \end{subfigure}\\ 
  \begin{subfigure}{0.7\textwidth}
    \centering
    \includegraphics[width=\linewidth]{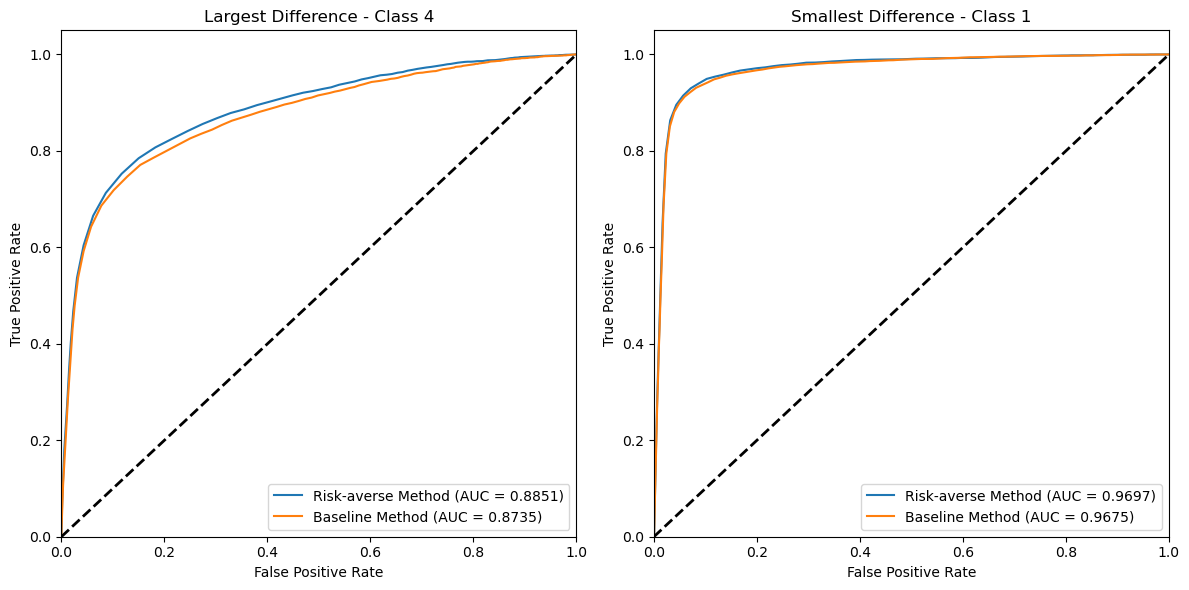}
    \caption{$\kappa=0.35$}
    \label{fig:auc-sub2}
  \end{subfigure}
\caption{For the 10-class experiment, we show the results with the largest and the smallest difference between the AUC scores.}
\label{fig:auc}
\end{figure}

For a more thorough comparison, we have also calculated the average ROC curve and AUC of each class over all the attempts. In the calculation of the function, one has to calculate the true positive rate (TPR) and the false positive rate (FPR) corresponding to different decision thresholds. In the Crammer-Singer method, the threshold $t$ is involved in the inequality
\[
\psi_i(x_i)-\max\big(\psi_j(x_i)\,|\,j\neq i\big)>t.
\]
In the general case, when deciding if an observation belongs to class $i$, the threshold is set to 0, which means that if an observation has the highest decision score with the classifier of the $i$th class, then it belongs to the $i$th class. In our experiment, we vary this threshold from its minimum possible value to its maximum, and for every class, we calculate its average TPR and NPR over all the attempts. Then we plot the ROC curve and report the AUC. Similar to any other metrics we use, all the calculations for both our method and the risk-neutral method are performed on the same training and test set in each attempt. 

Fig.\ref{fig:auc} shows the ROC and AUC of the best and the worst class based on the difference in the AUC scores. For the same value of $\kappa$, we consider the class with the largest difference between the AUC scores of the risk-averse method and the score of the risk-neutral method among the differences of all 10 classes. Likewise, we identify the class, for which this difference is the smallest. When the risk parameter $\kappa$ has too high value ($\kappa=0.95$), the difference between the two methods is very small, almost invisible. When we decrease its value to $\kappa=0.35$, the risk-averse method begins to perform better compared to the risk-neutral method, showing a larger difference in AUC values and a more significant gap between the ROC curves. Moreover, we see that no matter the risk level, even in the worst case, our method has a higher AUC score, meaning that for all 10 classes, our method designs a classifier that is better in terms of AUC.

\subsection{Non-linear Scalarization vs. Linear Scalarization}

We conducted tests to evaluate the difference between the nonlinear risk aggregation and the linear risk aggregation model. We repeat the experiments described in Section~\ref{sub:mislabeled_data} using the same protocol and compare the performance of the two aggregation methods. 

\begin{figure}[h!] 
\centering 
\includegraphics[width=0.8\linewidth]{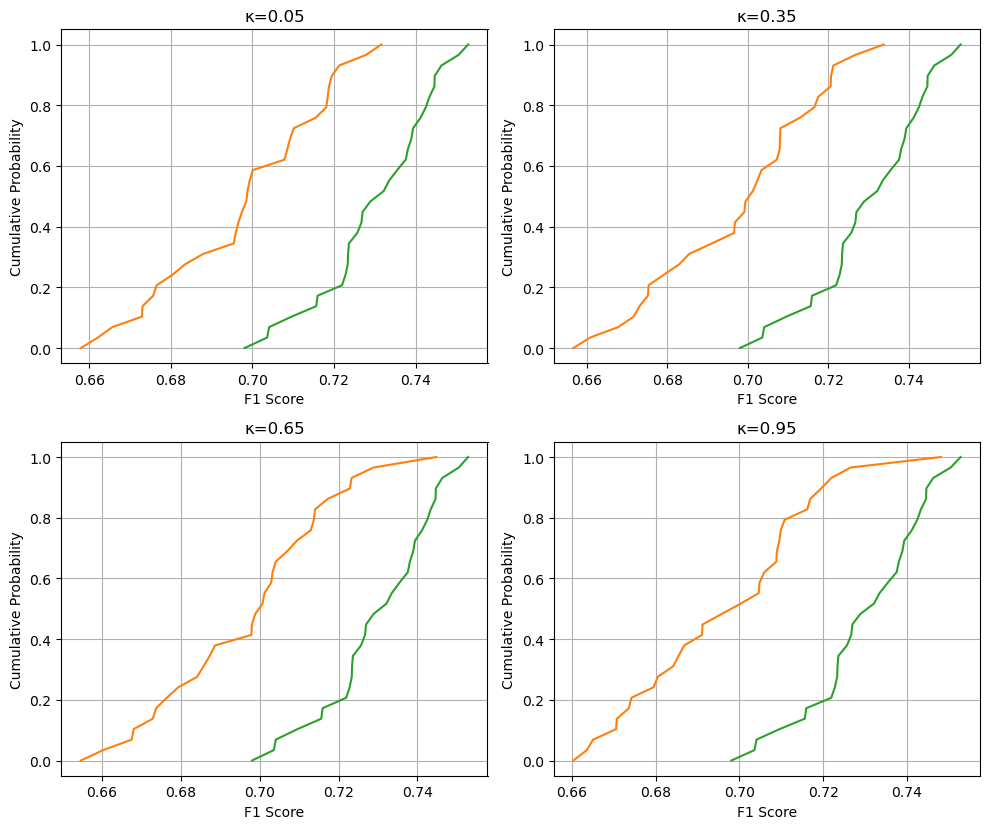} 
\caption{CDF of the average F1 scores of the non-linear aggregation model (green) and the linear aggregation model (orange) with different risk levels $\kappa$ using mislabeled data.} 
\label{fig:non_linear} 
\end{figure}

Fig.~\ref{fig:non_linear} reports results from 30 runs using the setup in Section~\ref{sub:mislabeled_data}. The outer risk measure is first-order MSD with $\kappa=0.05$, and the inner risk measure is first-order MSD with $\kappa=0.15$. The advantage of the nonlinear aggregation is clear from the graphs. Numerically, its average F1 score is $0.7297$, whereas the best of the four linear aggregation models achieves $0.6975$ score. Since in each of the 30 runs, we use the same split of the dataset for all models, a paired comparison shows that the nonlinear model outperforms the linear aggregation model in every run; the empirical probability of superiority is $1.00$. Moreover, we can see that in the graph, the CDF of the nonlinear model is less spread. Numerically, the nonlinear model's F1 score has a standard deviation of $0.0143$, while the best case among the linear ones is $0.0202$. We observe the same pattern under alternative experiment setups.This shows that not only does nonlinear aggregation gives better performance against risky situations, it is also provides greater stability.

\subsection{Risk-averse Kernel-based Method}

In this section, we report the experiments using the Electrical Fault detection and classification data from \cite{ElectricalFaultDetection2024} for testing the risk-averse kernel method. We first applied both risk-neutral Cramer-Singer's method and the traditional binary SVM method in a One-vs-All manner on the data. The result shows that the data is not linearly separable. Within 100 attempts, both methods only provide a similar average F1 score of around 0.55. We have solved the risk-averse dual kernel problem \eqref{p:dual-onestage} reformulating it for the mean-semideviation in the second stage and the total risk expectation in the first stage. We used the most popular kernels. The Guassian kernel (or RBF) and Laplacian kernel are more appropriate for this data set compared to other kernels such as polynomial kernel and cosine kernel. The Gaussian kernel is defined as 
$K(x,y) = \exp (-\gamma \| x-y\|_{2}^2)$ with $\gamma = \frac{1}{2\sigma^2}$, where $\sigma>0$ is the standard deviation. The Laplacian kernel is given by $K(x,y)=\exp (-\gamma \|x-y\|_{1})$. 

In Fig.\ref{fig:kernel-result}, we report the CDF of the average F1-scores in 100 attempts using the two kernels with different parameter values. As we can see, both kernels with appropriate parameter values, can produce a remarkable classification result. Compared to the F1 0.55 score from the non-kernel method. This shows that our new formulation using systemic risk measures is suitable also for kernel-based methods. Furthermore, the figures indicate that the F1-score of the risk-averse method is stochastically larger (with respect to the first-order stochastic dominance) than the baseline model. 

\begin{figure}[h]
\centering
  \begin{subfigure}{0.49\textwidth}
    \centering
    \includegraphics[width=\linewidth]{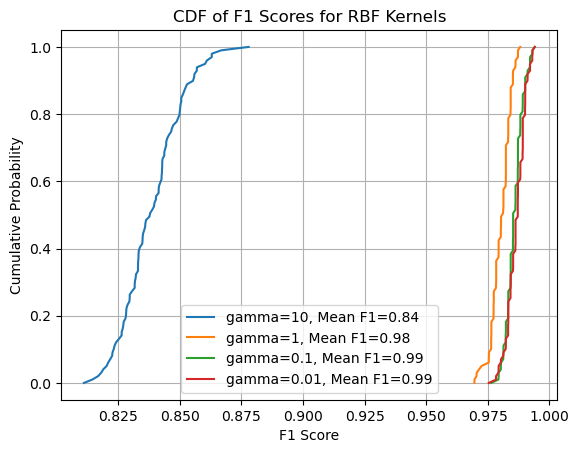}
    \label{fig:kernel-sub1}
  \end{subfigure} 
  \begin{subfigure}{0.49\textwidth}
    \centering
    \includegraphics[width=\linewidth]{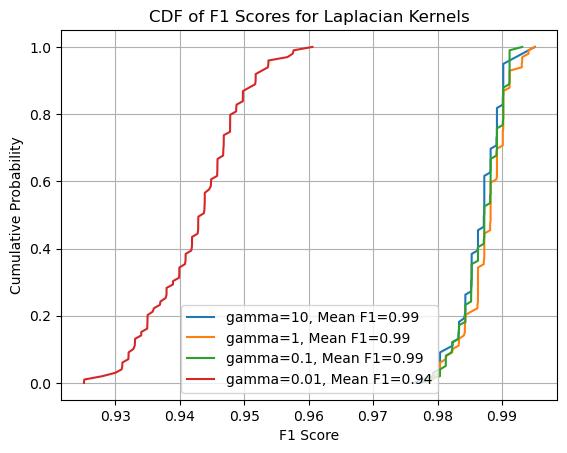}
    \label{fig:kernel-sub2}
  \end{subfigure}
\caption{Results with the Gausian (RBF) and the Laplacian kernels.}
\label{fig:kernel-result}
\end{figure}

We also conducted numerical experiments where the kernel version of the Cramer–Singer method is used in order to compare it with the proposed risk-averse classification design.  We set the risk parameter $\kappa=0.3$. The parameters for each method are optimized separately. The experiment setup is the same as in section~\ref{sub:mislabeled_data} but with the dataset we use in this section. The results are displayed in Fig.~\ref{fig:kernel-comparison}; the cumulative distribution functions are based on 100 runs.

\begin{figure}[h]
\centering
 \includegraphics[width=0.6\textwidth]{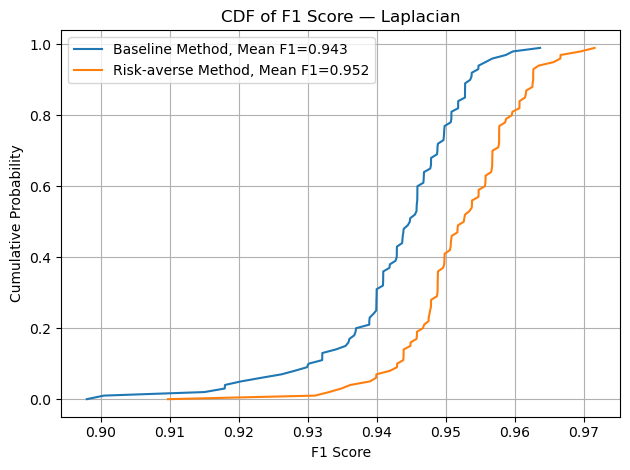}
\caption{Results with the Laplacian kernel for the ris-neutral and risk-averse classifiers.}
\label{fig:kernel-comparison}
\end{figure}

\subsection{Fair Risk-averse Classification}
\label{sub:fair_risk_averse_classification}

We have conducted experiments on data regarding equal opportunity fairness using the two-stage formulation. The data we utilized is the Drug Consumption Dataset \cite{Dua2019} from the UCI Machine Learning Repository. The Drug Consumption (quantified) dataset consists of data on the consumption of 18 different drugs among 1,885 participants. This dataset includes demographic information such as age, gender, education level, and country of residence.  Additionally, it contains scores on various personality traits, including neuroticism, extroversion, and openness to experience, measured using the NEO-FFI-R questionnaire. This is a 60-item questionnaire that measures the five major personality traits: neuroticism, extroversion, openness, agreeableness, and conscientiousness . 
This dataset contains usage of many different drugs, and for each drug, the consumption is represented by a categorical variable, indicating the frequency of use, starting from `Never used', `Used over a decade ago', to 'Used in the last day'. In order to test our SVM model and compare it with other baseline methods, we need to make the target binary. We slightly change the objective to classifying the drug consumption into 'Used in the last year' and 'Not used in the last year'. 

In the following experiment, we will use the two metrics: EO-difference and EO-ratio, to evaluate and compare the fairness between the different models.
Apart from fairness, we use F1-score to evaluate the performance of the models. For the two-stage model with 4 classes, we calculate the F1-score of the two original classes (`Used' and `NotUsed') because correctly classifying gender is not our goal.

In our experiments, we compare our classification method with nonlinear aggregation of contextual risk (\textbf{CNACR}) with \textbf{HDRFC}~\cite{wang2024wasserstein}, a Wasserstein distributionally robust classifier with fairness constraints, and with a binary soft-margin \textbf{SVM}. The baseline HDRFC also incorporates robustness and fairness together just like CNACR, making it a natural baseline for our experiment. For each run, we randomly sample $30\%$ of the dataset as a clean test set. From the remaining $70\%$, we randomly flip the binary gender for $20\%$ of examples, and then split this corrupted remainder into training and validation with a $4{:}1$ ratio. For each method, we perform a hyperparameter grid on the training set and select models on the validation set by the fairness target $\mathrm{EO\text{-}ratio} \ge 0.90$; among feasible configurations we choose the one with the highest F1 (falling back to the largest EO-ratio if none meet the target). We repeat this procedure for $20$ independent runs and report mean $\pm$ standard deviation.

\begin{table}[ht]
\centering
\setlength{\tabcolsep}{8pt}
\begin{tabular}{lcc}
\toprule
\textbf{Method} & \textbf{F1-score} & \textbf{EO-ratio} \\
\midrule
CNACR (ours) & $0.8099 \pm 0.0164$ & $0.8959 \pm 0.0605$ \\
HDRFC        & $0.7934 \pm 0.0206$ & $0.8692 \pm 0.0596$ \\
SVM          & $0.8151 \pm 0.0134$ & $0.7884 \pm 0.0504$ \\
\bottomrule
\end{tabular}
\caption{Test performance (mean $\pm$ std over $20$ runs).}
\label{tab:fair}
\end{table}

We see in Table.~\ref{tab:fair}, CNACR attains the best EO-ratio while keeping F1 close to SVM, improving fairness substantially over plain SVM and exceeding HDRFC on both fairness and F1. 
Quantitatively, CNACR improves EO-ratio by $+10.75$ pp (percentage points) over SVM with a modest F1 decrease of $0.52$ pp, and outperforms HDRFC by $+2.67$ pp in EO-ratio and $+1.65$ pp in F1.
The small F1 gap relative to SVM is consistent with the typical fairness–performance trade-off reported in the literature (e.g., \cite{rychener2022metrizing}), whereas compared to HDRFC, CNACR achieves \emph{better} fairness with \emph{less} performance sacrifice.

\begin{figure}[ht]
\centering
\begin{subfigure}{0.95\textwidth}
  \centering
  \includegraphics[width=0.49\linewidth]{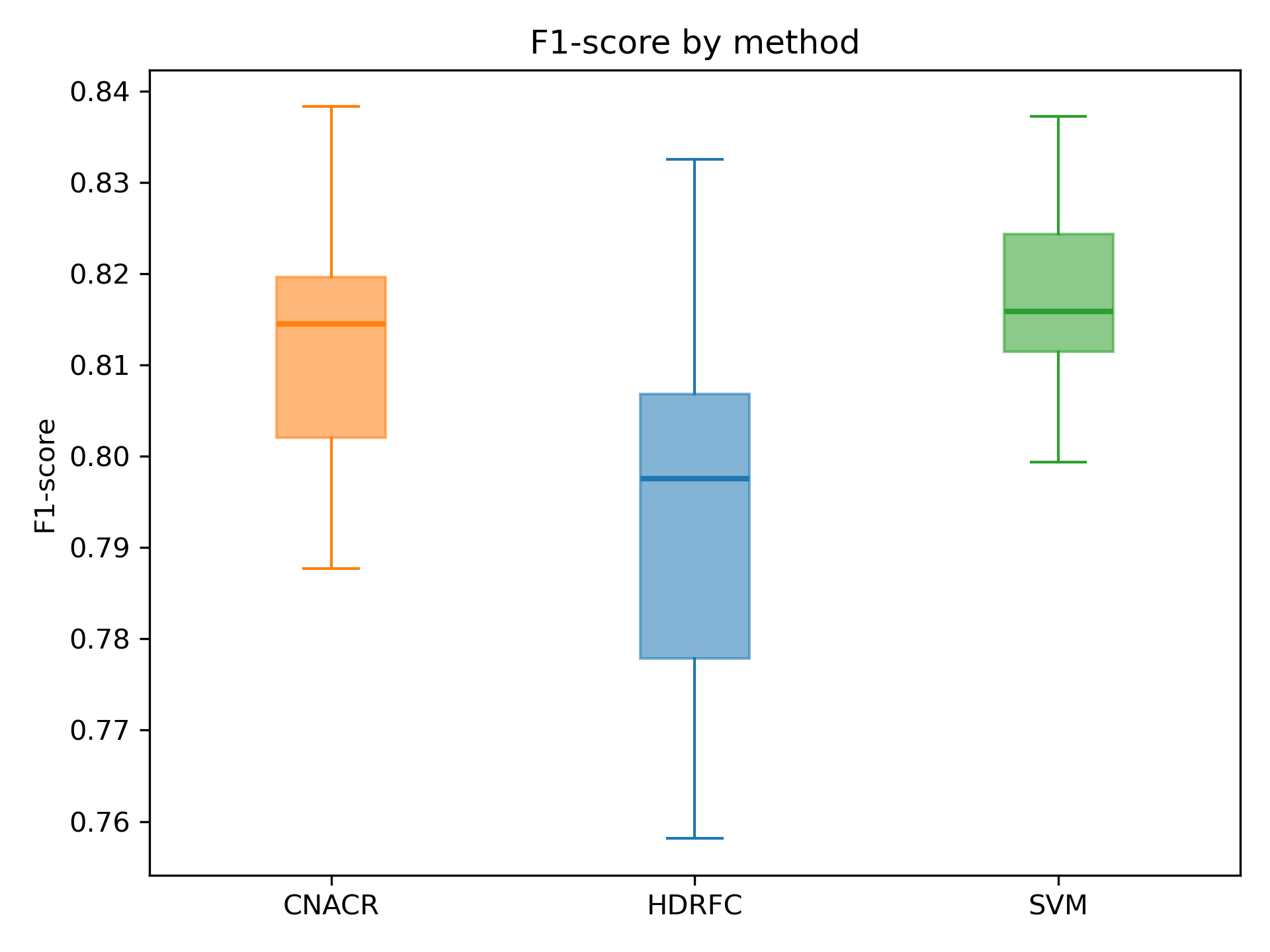}\hfill
  \includegraphics[width=0.49\linewidth]{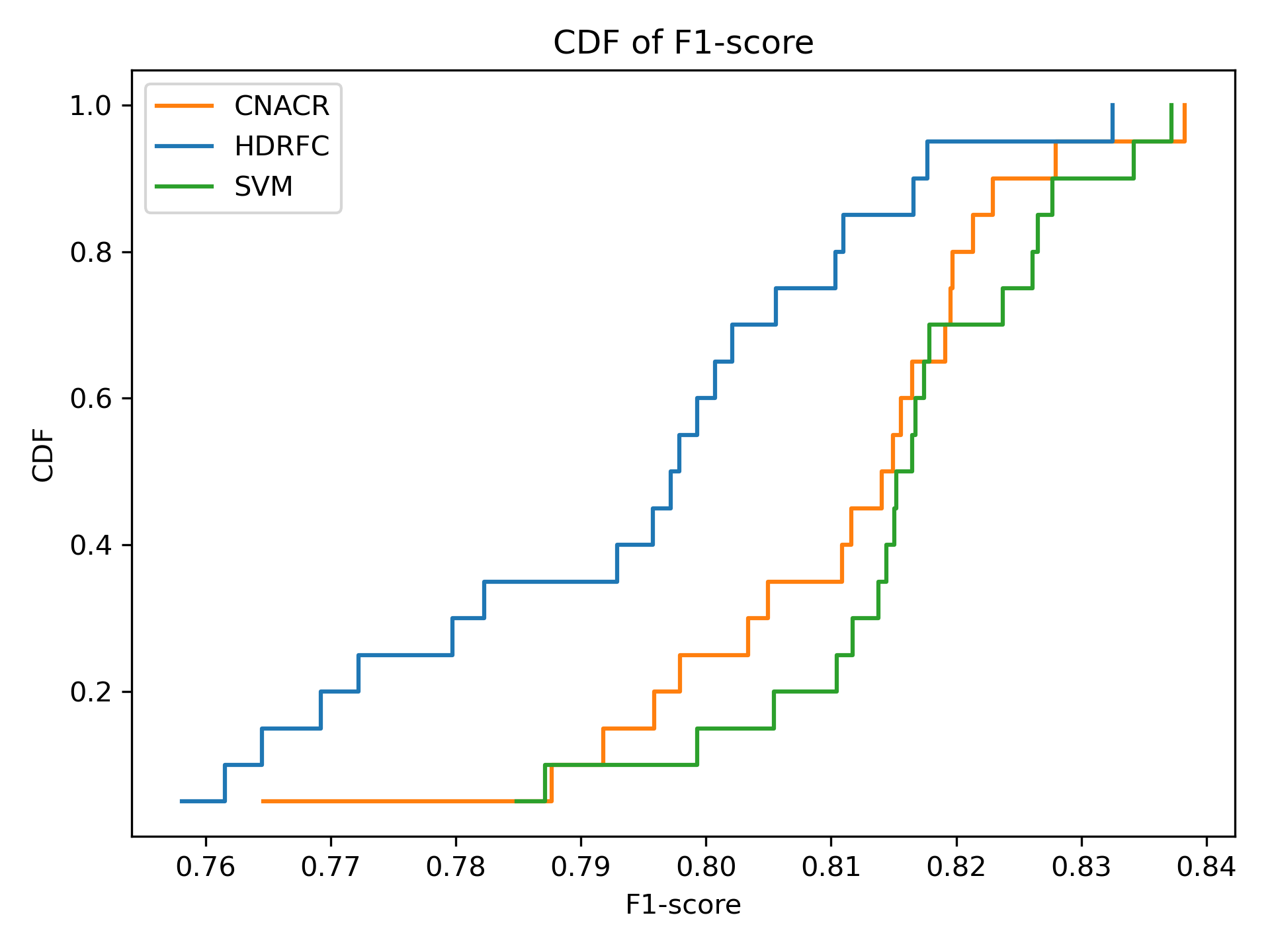}
  \caption{F1-scores}
  \label{fig:fair-f1}
\end{subfigure}

\begin{subfigure}{0.95\textwidth}
  \centering
  \includegraphics[width=0.49\linewidth]{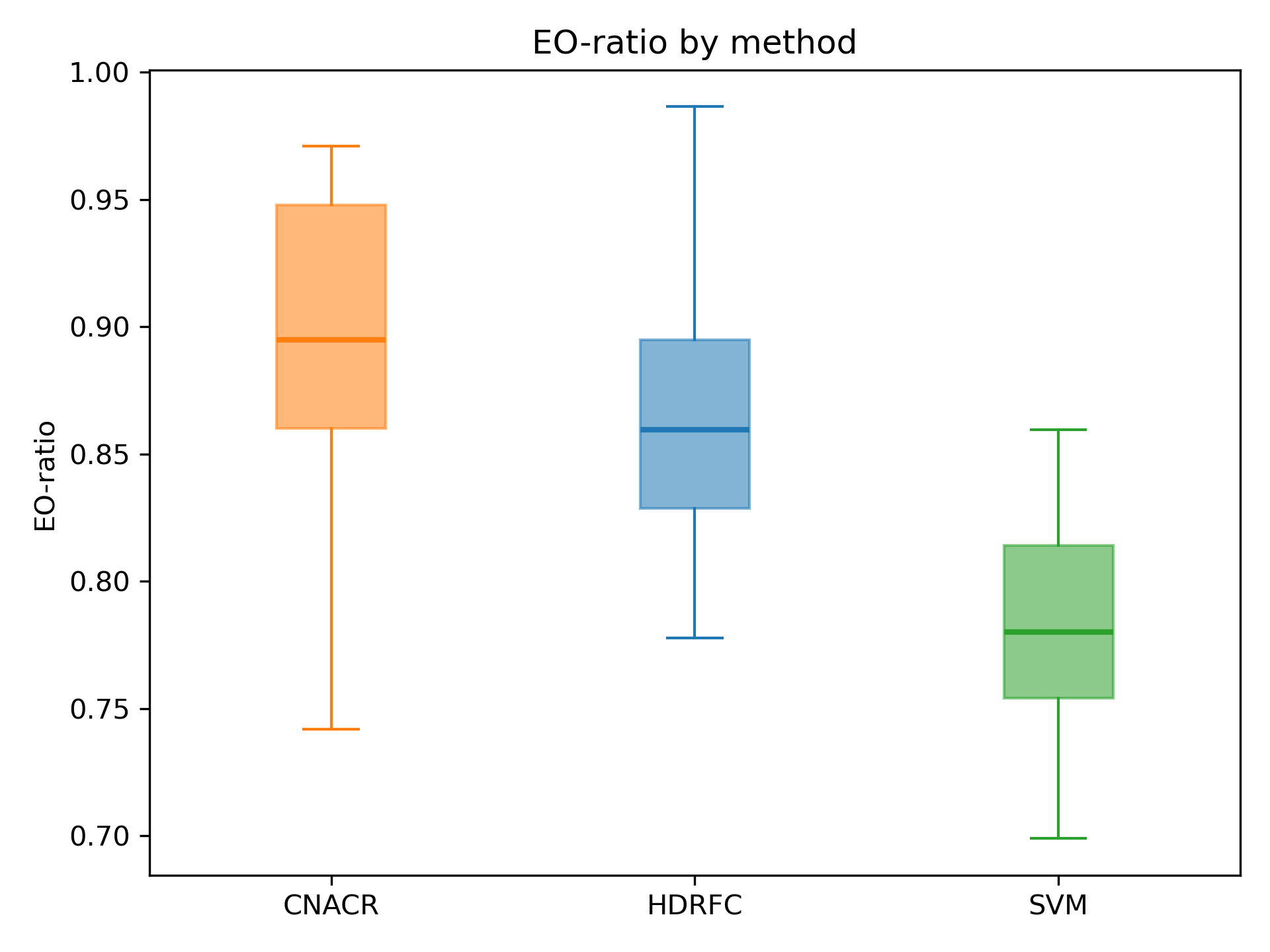}\hfill
  \includegraphics[width=0.49\linewidth]{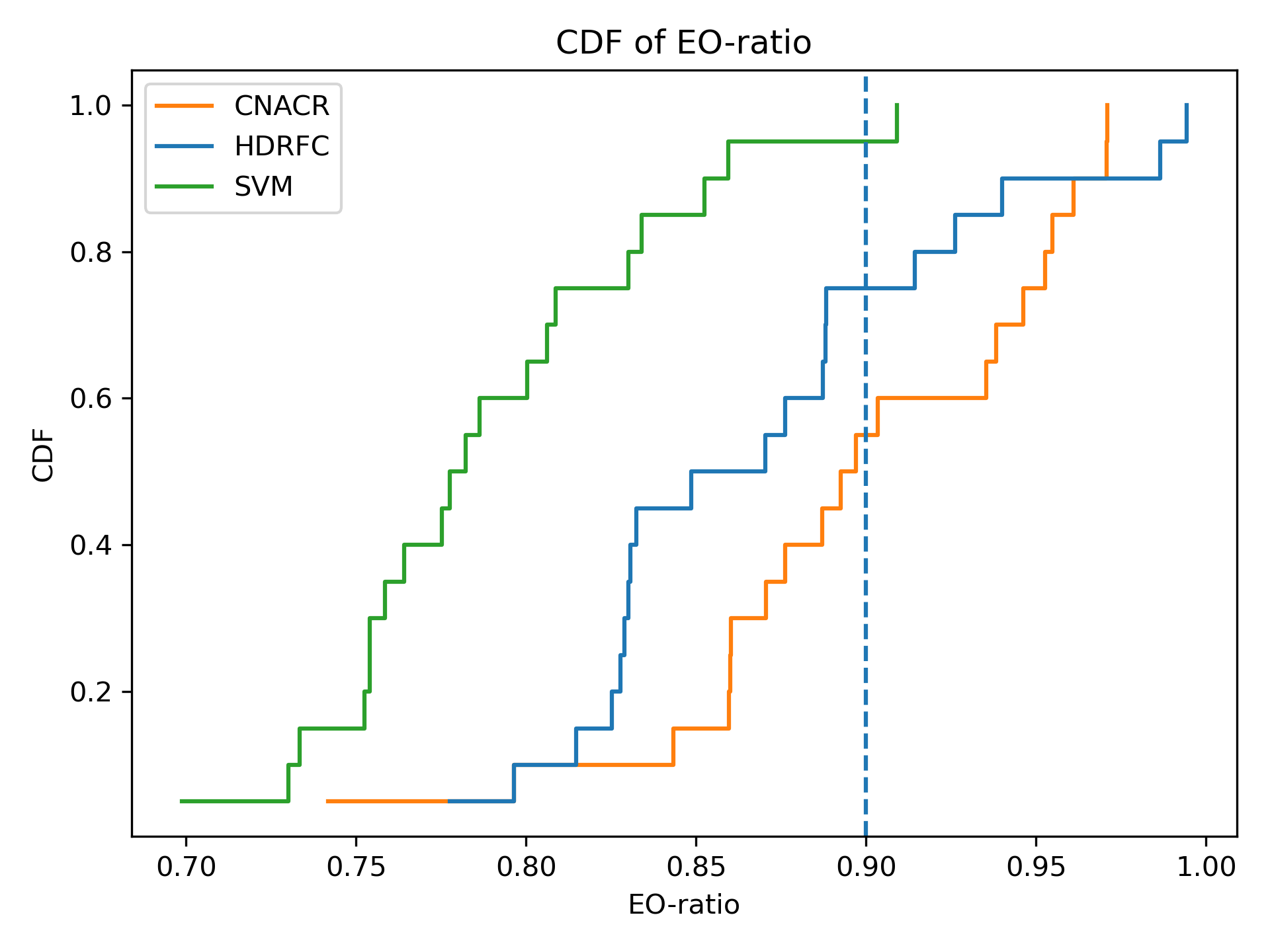}
  \caption{EO-ratios}
  \label{fig:fair-eo}
\end{subfigure}

\caption{Performance and fairness results.}
\label{fig:fair}
\end{figure}

Figure~\ref{fig:fair} shows a more direct comparison. From the box plot in Fig~\ref{fig:fair-f1}, we can see that CNACR has significantly higher average score compared to the baseline method HDRFC, with a much smaller variance, showing better generalization to the out-of-sample test set. The drawback of the performance compared to a plain SVM is much smaller than the precision drop by HDRFC. In the CDF plot, CNACR shows a first order stochastic dominance towards the HDRFC.
In Fig~\ref{fig:fair-eo}, CNACR again shows higher average in the box plot. Even though it also shows a higher variance, we can see in the CDF plot, the distribution of the CNACR's EO-ratio again dominates in the first order stochastic dominance the EO-ration of HDRFC. We notice that the confidence intervals of the fairness metrics are highly overlapping, so one may argue that this may not be enough to show a clear advantage. However, the design of this experiment mimics the real life usage of such fairness model, where we need our classifier prediction to be fair but also precise, i.e.,  we still need good performance of the model. The results show that our method can achieve the same or better level of fairness, while keeping a higher performance score compared to the baseline model.

\section{Conclusions}

In this paper, we have contributed to the application of risk-averse methods in classification by introducing a systemic point of view, which is particularly essential in multi-class scenarios. The involvement of a systemic measures of risk leads to a two-stage optimization problem in which the contextual risk of individual classes is calculated at the second stage and aggregated at the first stage where the optimal classifier is determined. We provide a tailored numerical method for solving the two-stage classification problem, which is highly efficient due to fact that the size of the optimization problems solved in the process does not increase with the size of the dataset. 

We have demonstrated that the risk-averse classification methods provide additional robustness to perturbation of the distributions, corrupted data, or small datasets; these methods generalize better to unknown data.  Furthermore, coherent systemic measures of risk allow us to enforce fairness without additional technical and computational burden. 
Empirical evidence shows that the improved performance over a risk-neutral counterpart becomes more pronounced when the number of classes increases. 

Additionally, we have extended the non-linear separation techniques based on kernels from the risk-neutral to the risk-averse case. We identify the dual problem of the two-stage optimization problem with systemic measure of risk and show that it has a structure of a two-stage problem as well.   Our numerical experiments confirm the efficiency and robustness of the kernel-based risk-averse method.

\end{document}